\documentclass[nohyperref]{article}

\usepackage{thmtools, thm-restate}
\usepackage{microtype}
\usepackage{graphicx}
\usepackage{booktabs} 
\usepackage{amsmath}
\DeclareMathOperator*{\argmin}{arg\,min}

\usepackage{hyperref}
\usepackage{algorithm}
\usepackage{algpseudocode}
\usepackage{bbm}
\usepackage{subcaption}
\usepackage{caption}
\usepackage{wrapfig}
\algtext*{Indent}
\algtext*{EndIndent}

\newcommand{\ind}{\perp\!\!\!\!\perp} 
\DeclareMathOperator\supp{supp}
\usepackage[multiple]{footmisc}

\usepackage[accepted]{icml2022}

\usepackage{amsmath}
\usepackage{amssymb}
\usepackage{mathtools}
\usepackage{amsthm}

\usepackage[capitalize,noabbrev]{cleveref}

\theoremstyle{plain}

\theoremstyle{definition}
\newtheorem{definition}{Definition}

\theoremstyle{remark}

\usepackage[textsize=tiny]{todonotes}

\icmltitlerunning{Influence-Augmented Local Simulators}

\begin{document}

\twocolumn[
\icmltitle{Influence-Augmented Local Simulators: \\ A Scalable Solution for Fast Deep RL in Large Networked Systems}



\icmlsetsymbol{equal}{*}

\begin{icmlauthorlist}
\icmlauthor{Miguel Suau}{delft}
\icmlauthor{Jinke He}{delft}
\icmlauthor{Matthijs T. J. Spaan}{delft}
\icmlauthor{Frans A. Oliehoek}{delft}
\end{icmlauthorlist}

\icmlaffiliation{delft}{Delft University of Technology}

\icmlcorrespondingauthor{Miguel Suau}{m.suadecastro@tudelft.nl}
\icmlkeywords{Simulation, Influence, Deep Reinforcement Learning.}

\vskip 0.3in
]



\printAffiliationsAndNotice{}  

\begin{abstract}
Learning effective policies for real-world problems is still an open challenge for the field of reinforcement learning (RL). The main limitation being the amount of data needed and the pace at which that data can be obtained. In this paper, we study how to build lightweight simulators of complicated systems that can run sufficiently fast for deep RL to be applicable. We focus on domains where agents interact with a reduced portion of a larger environment while still being affected by the global dynamics. Our method combines the use of local simulators with learned models that mimic the influence of the global system. The experiments reveal that incorporating this idea into the deep RL workflow can considerably accelerate the training process and presents several opportunities for the future.
\end{abstract}

\section{Introduction}
The remarkable success of Deep Reinforcement leaning (RL) on paper is in sharp contrast with its narrow applicability to real-world problems. Among many other reasons, the most important factor preventing the practical deployment of this framework is perhaps its high sample complexity \citep{botvinick2019reinforcement}. This is a very well-known issue and there is a long list of previous works that in one way or another have tried to alleviate it \citep{kakade2003sample, Mnih15Nature,NEURIPS2018_2de5d166}. Nonetheless, the framework is still far from being useful in practice. Here,
rather than proposing yet another method that tries to solve the problem directly, we present a more pragmatic approach to get around it. Our solution is based on the observation that Deep RL's best results have been obtained in domains like video games \citep{bellemare13arcade, vinyals2019grandmaster} or simulated environments \citep{openaigym,Ganesh2019ReinforcementLF, bellemare2020autonomous} where data collection is  extremely fast. Unfortunately, real-world problems are typically more complex and simulators, if available, are usually very slow \citep{dulac2019challenges}. 

In this work, we design lightweight versions of large simulators with the goal of speeding up the overall training process. The method we propose applies to domains where agents only interact with a reduced local part of a larger environment, yet they are indirectly being affected by the global dynamics. Traffic control is one example of such environments. Say, for instance, that we wanted to train an agent to control the traffic lights of a particular intersection in a very large city. To do so we could build a small local simulator that captures only the information that is directly relevant to the agent (traffic density in the neighborhood; \citealt{VanDerPol16NIPSWS}). However, after training, we may find out that an agent that does very well in the small simulator, performs poorly in the real intersection. The performance gap would be caused by a data distribution shift \citep{quionero2009dataset, arjovsky2021out}. Even though the simulator might be able to closely mimic the local dynamics (i.e. cars moving within the intersection), it would fail to account for the interactions of the local neighborhood with the rest of the city. Thus, the agent learns a policy based on certain transition dynamics that turn out to be very different in the real world. 
Alternatively, we could try to model the dynamics of a sufficiently large portion of the city, but this would surely result in a very slow simulator.

One important property of the traffic domain is that, although the agent's local problem may be \emph{affected} by many external variables (traffic densities in other parts of the city), it is only \emph{directly influenced} by the road segments that connect the intersection with the rest of the city. Hence, we can simply monitor the traffic densities at these road segments since, from the agent's local perspective, they summarize the effect of all the external variables. This insight is not specific to the traffic domain. It is in fact common in networked systems (e.g. warehouse commissioning, \citealt{claes2017decentralised}; electrical power grids, \citealt{wang2021multi}; heating systems, \citealt{gupta2021energy}; telecommunication networks, \citealt{suau2021offline}) that interactions between different components occur through a limited number of variables. 

Supported by the formal framework of \emph{influence-based abstraction} (IBA) \citep{oliehoek2021sufficient}, we exploit the above property to build local simulators that mirror the response of the global system through the so called \emph{influence predictor}. 

\paragraph{Contributions} The main contribution of this paper is the integration of the IBA framework with the Deep RL workflow.
Our experiments reveal, that the combination of  local simulators and influence predictors can considerably accelerate the reinforcement learning process.
We also demonstrate both theoretically and empirically that the memory needs of the influence predictor are fully determined by the agent's memory capacity. Moreover, we study the impact that distribution shifts caused by changes in the agent's policy may have on the influence predictor, and explore how to prevent the model from picking up on spurious correlations that are not invariant across policies. Finally, we investigate the effect of transfer learning and show how inaccurate simulators might also be able to render effective policies.
\section{Related Work}\label{sec:relatedwork}
The problem of sample complexity has been extensively studied by the RL community. Among many others, the most promising solutions are: replaying previous experiences to make more efficient use of the available data \citep{Mnih15Nature, DBLP:journals/corr/SchaulQAS15,NEURIPS2019_1b742ae2}, or learning surrogate models of the environment dynamics \citep{Sutton90integratedarchitectures,NEURIPS2018_2de5d166, schrittwieser2020mastering, moerland2020model}. Yet, these techniques are only effective when provided enough real samples. If not, replay buffers might not be sufficient to obtain good policies and surrogate models might generalize poorly. An alternative is to train agents with synthetic data coming from a simulator. However, most real-world scenarios are excessively complex and simulators, if available, are computationally expensive \citep{dulac2019challenges}. Here we argue that building a simulator of the entire system is often unnecessary. In fact, as we explain in the following sections, in many situations we can get away by just modelling the dynamics around the agent's local neighborhood.

A few prior works have investigated the computational benefits of factorizing large systems into independent local regions \citep{Nair05AAAI,  Varakantham07AAMAS, Kumar11IJCAI, Witwicki11AAMAS}. Unfortunately,  since local regions are often coupled to one another, such factorizations are not always appropriate. Nonetheless, in many cases, the interactions between regions occur through a limited number of variables. Using this property, the theoretical work by \citet{oliehoek2021sufficient} on influence-based abstraction (IBA) describes how to
build influence-augmented local simulators (IALS) of local-POMDPs, which model only the variables in the environment that are directly relevant to the agent while monitoring the response of the rest of the system with the influence prediction.  The problem is, the exact computation of the conditional influence distribution is intractable and we can only try to estimate it from data. \citet{Congeduti21AAMAS} provide theoretical bounds on the value loss when planning with approximate influence predictors. The work by \citet{NEURIPS2020_2e6d9c60} has empirically demonstrated the advantage of this approach to improve the efficiency of online planning in two discrete toy problems. 

Here, we extend the method to high dimensional problems and study how to integrate the IBA framework with Deep RL. This has profound implications that do not arise in the planning context, namely the relation between the agent's memory capacity and the history dependence of the influence predictor (Section \ref{sec:history_dependence}), and the problem of off-policy generalization (Section \ref{sec:policy_dependence}). Moreover, while \citeauthor{NEURIPS2020_2e6d9c60} showed that the IALS outperforms the global simulator only when the time budget is limited, our results reveal that the IALS can train policies in a fraction of the time and that these can match the same performance as policies trained on the GS, without imposing any time constraints, and despite the IALS is only approximate. 

\section{Preliminaries}\label{sec:background}
Here we introduce the notation used throughout the paper, provide the mathematical definition of the problem, and describe the Influence-based abstraction (IBA) framework \citep{oliehoek2021sufficient}, which gives theoretical support to the method we introduce in Section \ref{sec:ILS}. 
\begin{figure*}[ht]
\hspace{12mm}
\includegraphics[width=0.85\textwidth]{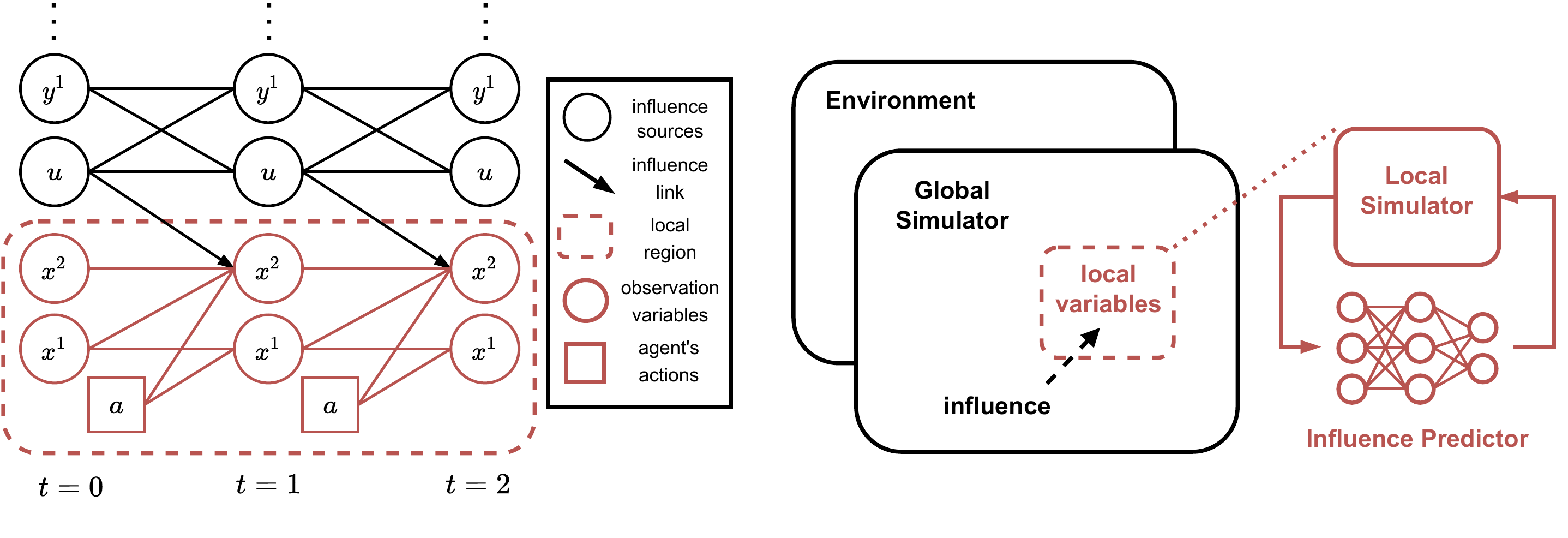}
\vskip -0.2in
\caption{Left: A Dynamic Bayesian Network of a Local-POMDP unrolled 3 timesteps. Right: A diagram of the IALS.}
\label{fig:pomdp+ips}
\end{figure*}

\subsection{Problem definition}
As explained in the introduction, we target domains where, although the agent is affected by the global dynamics, it can only observe its local neighborhood. These can be modelled as partially observable Markov decision processes (POMDP) \citep{Kaelbling96JAIR}.

\begin{definition}[POMDP]
A partially observable Markov decision process (POMDP) is a tuple $\langle S,A,T,R,\Omega,O\rangle$ where $S$ is the state space, $A$ is the set of actions, $T$ is the transition probability function, with $T(s_{t+1}|s_t,a_t)$, $R(s_t,a_t)$ defines the immediate reward, $\Omega$ is the observation space and $O$ is the observation probability function, $O(o_{t+1}|s_{t+1})$.
\end{definition}

As usual, the task consists in finding the policy $\pi$ that maximizes the expected discounted sum of rewards \citep{SuttonBarto98}. Since the agent receives only a local observation $o$ of the true state $s$, a policy that is based only on the most recent information can be sub-optimal. In general, the agent is required to keep track of its past actions and observations to make the right action choices. Policies are therefore mappings from the past action-observation history (AOH), $h_t = \langle o_1, a_1 ..., a_{t-1}, o_t\rangle$, to actions, $\pi(a_t|h_t)$.  

We consider, in particular, a special case of a POMDP hereinafter referred to as a \emph{Local-POMDP}.
\begin{definition}[Local-POMDP]
\label{def:local-pomdp}
A Local-POMDP is a POMDP where $O$ and $R$ depend only on a subset of state variables $x \in X \subseteq S$, which define the agent's local region. Such that $O(o_{t+1}|s_{t+1}) = \dot{O}(o_{t+1}|x_{t+1})$ and $R(s_t, a_t) = \dot{R}(x_t, a_t)$, where $\dot{O}$ and $\dot{R}$ are the local observation and reward functions.
\end{definition}



Looking at the above definition, one can argue that simulating the transitions $T(s_{t+1}|s_t, a_t)$ of the full set of state variables
is unnecessary, and while doing so is possible in small problems, it might become computationally intractable in large domains. We can instead define a new transition function $\bar{T}$ that models only the local state variables $x$, $\bar{T}(x_{t+1}|x_t, a_t)$.
The problem is that $x_{t+1}$ may still depend on the rest of state variables $s_t$ and thus $\bar{T}$ is not well defined.
Hence, we are forced to keep 
track of the action-local-state history (ALSH) $l_t = \langle x_1, a_1 ..., a_{t-1}, x_t\rangle$
such that we can sample the next local state as
\begin{equation}
\small
\begin{split}
    P(x_{t+1}|h_t, a_t) = \sum_{s_t}T(x_{t+1}|s_t, a_t) P(s_t|l_t).
    \label{eq:local_transitions}
\end{split}
\end{equation}
\subsection{Influence-based Abstraction}
With the above formulation, the problem is only partially solved.
We no longer need to simulate the global state transitions, yet we have to maintain a distribution over the full state given the ALSH, $P(s_t|l_t)$, and then calculate the local transitions with $P(x_{t+1}|s_t, a_t)$. As explained in the introduction, in many local POMDP problems, however, only a fraction of the state variables will \emph{directly influence} the local region.

The diagram in Figure \ref{fig:pomdp+ips} is a Dynamic Bayesian Network (DBN) \citep{pearl88, boutilier1999decision} of a local POMDP prototype. The agent's local region, 
corresponds to the variables, denoted by $x \in X$, that lie within the red box, $x = \langle x^1,x^2 \rangle$. The diagram also shows the non-local variables, known as influence sources $u \in U \subseteq S \setminus X$, that influence the local region directly. The three dots on the top indicate that there can be, potentially many, other non-local variables in $S$. These are denoted by $y$ and, as shown in the diagram, can only influence $x$ via $u$. As such, given $u_t$, $x_{t+1}$ is conditionally independent of $y_t$, $P(x_{t+1}|x_t,u_t,y_t) = P(x_{t+1}|x_t, u_t)$. 
\begin{definition}[IALM] An influence-augmented local Model (IALM) is a POMDP with local states $x_t \in X$, influence sources $u_t \in U$, local observation function $\dot{O}(o_{t+1}|x_{t+1})$, local reward function $\dot{R}(x_t, a_t)$, local transition function $\dot{T}(x_{t+1}|x_t, u_t, a_t)$ and an influence distribution $I(u_t|l_t)$.
\end{definition}

Using the IALM we can simulate the local transitions as
\begin{equation}
\small
\begin{split}
    P(x_{t+1}| l_t, a_t) = 
    \sum_{u_t}  \dot{T}(x_{t+1}|x_t, u_t, a_t)I(u_t|l_t).
    \label{eq:IALM}
\end{split}
\end{equation}
Note that, as opposed to the Local-POMDP, the transition function $\dot{T}$ in the IALM is defined purely in terms of the local state variables and the influence sources. Moreover, since $u_t$ \emph{d-separates} \citep{Bishop06book} $x_t$ from $y_t$, the influence distribution is just the conditional probability over $u_t$ instead of the full set of of state variables $s_t$. All in all, this translates into a much more compact, yet \emph{exact} representation of the problem \citep{oliehoek2021sufficient}, which should be computationally much lighter than the original POMDP.



\section{Influence-Augmented Local Simulators for Deep RL}\label{sec:ILS}
In the following, we describe how we make use of the IALM formulation to design lightweight simulators that can speed up the long training times imposed by
 neural network policies. Figure \ref{fig:pomdp+ips} (right) shows a diagram of the influence-augmented local simulator (IALS), which is composed of a \emph{local simulator} and an \emph{approximate influence predictor}.


\paragraph{Local simulator (LS):}
The LS is an abstracted version of the environment that only models a small portion of it. As opposed to a global simulator (GS), which should closely reproduce the dynamics of every state variable, the LS focuses on characterizing the transitions of those variables $x_t$ that the agent directly interacts with, $\dot{T}(x_{t+1}|x_t, u_t, a_t)$. Recall that, in our setting, both the reward $R$ and  observation $O$ functions are defined in terms of $x_t$ and $a_t$.

\paragraph{Approximate influence predictor (AIP): } The AIP monitors the interactions between the local region and the external variables $y_t$ by estimating $I(u_t|l_t)$. Ideally, we would like the approximation to match the true influence distribution. However, 
due to combinatorial explosion, computing the exact probability $I(u_t|l_t)$ is generally intractable \cite{oliehoek2021sufficient}, and so we can only try to approximate it using, for instance, a parametric function. We write  $\hat{I}_\theta$ to denote the AIP,
where $\theta$ are the parameters, which need to be learned from data. Replacing the true influence distribution with an approximation implies that we are no longer guaranteed to find the optimal policy \cite{Congeduti21AAMAS}. Nonetheless, as we show in our experiments, it is often worth trading accuracy for computational efficiency.
We model $\hat{I}_\theta$ using a neural network, which we train on a dataset of $N$ samples of the form $(l_n, u_n)$ collected from the GS (see Algorithm \ref{alg:collect} in Appendix \ref{ap:algorithms}).
Since role of the AIP is to estimate the conditional probability of the influence sources $u_t$ given the past ALSH,
we can formulate the task as a classification problem. The neural network is optimized using the expected cross-entropy loss, 
\begin{equation}
\small
     L(\hat{I}_\theta) = - \frac{1}{N}\sum_{n=0}^N u_n \log \hat{I}_\theta(u_n|l_n),
\label{eq:loss}
\end{equation}
which minimizes the KL divergence between the true probabilities $I(u_t|l_t)$ and those predicted by $\hat{I}_\theta(u_t|l_t)$\footnote{Please refer to Appendix \ref{ap:imp_details} for more details on the implementation of the AIP.}.
Once the network has been trained, we can simulate trajectories with the IALS (Algorithm \ref{alg:sample} in Appendix \ref{ap:algorithms}). These trajectories can then be used to train policies with any standard Deep RL method \citep{Mnih15Nature, schulman2017proximal}.

In the next two sections we discuss two important considerations when training AIPs for the RL setting. We first show that, since the agent's memory is inevitably finite, we normally will not need to train complicated AIPs that condition on the full ALSH (Section \ref{sec:history_dependence}). In Section \ref{sec:policy_dependence},
we study the impact that distribution shifts caused by changes in the agent's policy may have on the AIP, and explore how to prevent the AIP from picking up on spurious correlations that are not invariant across policies.

\subsection{Finite memory agents and AIP history dependence }\label{sec:history_dependence}

As mentioned before, we can only rely on approximations of $I(u_t|l_t)$ since computing the exact distribution that conditions on the full ALSH is intractable. Unfortunately though, even with the most sophisticated RNNs \cite{hochreiter1997long, Cho2014Learning}, learning long term temporal dependencies is also very difficult. However, it is worth noting that, if capturing long term dependencies is hard for the AIP, it is too for the agent. In fact, it is very common in Deep RL to find policies that have access only to finite memories \cite{Mnih15Nature, Oh16ICML} or that are trained on short sequences \cite{Schmidhuber91NIPS,hausknecht2015deep}. Thus, if the agents memory is finite, one might well wonder whether the extra level of accuracy that an influence predictor which conditions on the full ALSHs provides is needed. 
The result below shows that, the history dependence of the influence predictor is,  under mild conditions, determined by the agent's memory capacity. 
\begin{restatable}{theorem}{history}
Let $l_{-k:t}$ be a truncated version of $l_t$ containing only the last $k$ action-local-state pairs and let $a_{0:k}$ be the sequence of past actions from time $0$ up to $k$. Let the agent's policy $\bar{\pi}$ be a function of $h_{-k:t}$ $\bar{\pi}(a_t|h_{-k:t})$, where $h_{-k:t}$ is also a truncated version of $h_t$. If for all $u_t$ we have $P(u_t|l_{-k:t},a_{0:k}) = P(u_t|l_{-k:t})$, then an influence predictor that conditions only on $l_{-k:t}$, $\bar{I}(u_t|l_{-k:t})$, is sufficient to guarantee no loss in value\footnote{Note that this refers only to the value of polices of the form $\bar{\pi}(a_t|h_{-k:t})$, which might perform arbitrarily worse than policies that condition on the full AOH, $\pi(a_t|h_t)$ \cite{Singh94ICML}.}.
\label{thm:history_dependence}
\vskip -0.5 in
\end{restatable}
\renewcommand\qedsymbol{}
\begin{proof} 
Found in Appendix \ref{ap:proofs}.
\end{proof}
\vskip -0.1 in
Intuitively, a finite memory agent whose policy conditions on the last $k$ elements in the AOH is only be capable of computing an expectation over the action-values given these $k$ elements. In turn, the distribution of $u_t$ given the full ALSH $l_t$, $I(u_t|l_t)$, is effectively ``washed out'' by the same expectation.
The upshot is that the agent's memory capacity limits the temporal dependencies of the influence predictor\footnote{More details on the practical implications of this result are given in Appendix \ref{ap:imp_details}.}, meaning that $I$ may as well condition directly on $l_{-k:t}$ rather than $l_t$. This insight is empirically evaluated in Section \ref{sec:exp_history_dependence}.

\subsection{Off-policy generalization and d-sets}\label{sec:policy_dependence}
Given that the \emph{true influence distribution} conditions on the full ALSH, it is, in principle, independent of the exploratory policy $\pi_0$ that we use to collect the data. Indeed, if we would represent $I(u_t|l_t)$ with a table, we could compute unbiased estimates of the true probabilities $I(u_t|l_t)$ by using any arbitrary policy that visits every ALSH infinitely often. The problem arises when we use function approximators to estimate the influence sources. In particular, if we call  $P^{\pi_0}(l_t,u_t)$ the stationary joint distribution of influence sources and ALSHs induced by the exploratory policy $\pi_0$ that we use to collect the data $D^{\pi_0} = \left\{(l_1,u_1),... (l_N,u_N)\right\}$ from the GS (as described by Algorithm \ref{alg:collect} in Appendix \ref{ap:algorithms}), we can write
\begin{equation}
\small
    \theta^* = \argmin_\theta L(\hat{I}_\theta, D^{\pi_0}),
    \label{eq:loss_datadist}
\end{equation}
where we see that, in fact, the loss in \eqref{eq:loss} depends on $D^{\pi_0}$. Therefore, because we are fitting $\hat{I}_\theta$ to $D^{\pi_0}$, the model will be biased towards $P^{\pi_0}$. This is not so bad considering that we want the influence predictions to be more reliable for those ALSHs that the agent visits more often. However, it can pose a problem when the policy $\pi$ that we train starts deviating from $\pi_0$. In general, we want the AIP to perform  well on any distribution that the agent may experience during training, $\{P^\pi\}_{\pi \in \Pi}$,
where $\Pi$ denotes the set of possible policies. We then rewrite the optimization problem in \eqref{eq:loss_datadist} as
\begin{equation}
\small
    \theta^* = \argmin_\theta \max_{\pi \in \Pi} L(\hat{I}_\theta, D^{\pi})
\label{eq:00D}
\end{equation}
to indicate that we wish to minimize the worst-case loss.

The key issue here is that, before training, we have only access to the exploratory policy $\pi_0$. Equation \eqref{eq:00D} is an instance of a well-studied problem in the field of out-of-distribution generalization \cite{quionero2009dataset,arjovsky2021out}. In principle, according to \citet[Section~3.6.1]{arjovsky2021out}, finding a solution to \eqref{eq:00D} by minimizing \eqref{eq:loss_datadist} requires, \textbf{(i)} $\supp(P^{\pi}) \subseteq \supp(P^{\pi_0})$, \textbf{(ii)} infinite number of samples, and \textbf{(iii)} infinite capacity models.

The first condition is met so long as $\pi_0(a_t | l_t) > 0$ for all $a_t$ and $l_t$. The second and third conditions, on the other hand, can never be fulfilled in practice. On top of that, high-capacity influence predictors are particularly undesirable for our purpose because they are computationally demanding. The consequence of training low-capacity models on finite datasets is that they may pick-up on spurious correlations \cite{Pearl2000Causality} between ALSHs and influence sources\footnote{\label{fn:example}A visual example of a spurious correlation appearing in the traffic problem is given in Appendix \ref{ap:example}.}. These correlations could be just an artifact of $\pi_0$ and may disappear after the policy is updated. One way to prevent the above, is to find a representation of $l_t$ that elicits an invariant predictor of $u_t$ across all $P^\pi$ \cite{peters2016causal,arjovsky2019invariant, krueger2021ood}.
\begin{definition}[invariant predictor]
A subset of variables $d_t$ from $l_t$ elicit and invariant predictor of $u_t$ across policies  $\pi \in \Pi$ if for all $d_t$ in the intersection of supports, $\supp(P^{\pi_1}(d, u)) \cap \supp(P^{\pi_2}(d, u))$, we have
\begin{equation}
\small
    P^{\pi_1}(u_t| d_t) = P^{\pi_2}(u_t|d_t) \quad \text{for all} \quad \pi_1, \pi_2 \in \Pi.
\end{equation}
\end{definition}
The notion of d-separation \cite{Bishop06book} comes in handy here. Given a DBN, such as the one depicted in Figure \ref{fig:pomdp+ips}, one can determine a subset of variables in the ALSH that is sufficient to predict $u_t$ \cite{oliehoek2021sufficient}.
\begin{definition}[d-separating set] The d-separating set (d-set) is a subset of variables $d_t$ from $l_t$, such that influence sources $u_t$ and the remaining parts of the ALSH $l_t \setminus d_t$ are conditionally independent given $d_t$, $ (u_t \ind l_t \setminus d_t \mid d_t) $.  The d-set is said to be minimal $d^*_t$, when no more variables can be removed from $d^*_t$ while the above still holds.
\label{def:dset}
\end{definition}
\begin{restatable}{theorem}{invariant}
A subset of variables $d_t$ from $l_t$ is only guaranteed to elicit an invariant predictor of $u_t$, across all $\pi \in \Pi$, if (i) $d_t$ constitutes a d-separating set and (ii) all policies are functions of the local AOH, $\pi(h_t)$.\label{thm:invariant}
\end{restatable}
\begin{proof}
Found in Appendix \ref{ap:proofs}.
\end{proof}
\vskip -0.1 in
Note that, according to the result above, the full ALSH $l_t$ does elicit an invariant predictor of $u_t$ since it is, by definition, a d-set. Hence, feeding $l_t$ should be sufficient for the model to find stable and invariant correlations. The problem, however, is that in practice, models tend to converge to low-capacity solutions that require little ``effort" to learn (i.e. least-norm solution) \cite{arjovsky2021out}. As such, the representations formed by an influence predictor that is fed the full ALSH may or may not constitute a d-set. The solution we propose is to feed solely a minimal d-set $d_t^*$ into the AIP $\hat{I}_\theta$. Not only does this reduce the dimensionality of the input space but also prevents the AIP from learning shortcuts \cite{geirhos2020shortcut} resulting from spurious correlations that are not stable under policies other than $\pi_0$\footref{fn:example}. Provided we have the DBN of the problem, there are many algorithms available that can help us find a minimal d-set \cite{Acid1996Algorithm, tian1998finding}. If not, some domain knowledge may be sufficient in most cases, to remove a few variables from $l_t$ and prevent confounding.

Working in the joint space $(d_t, u_t)$ instead of $(l_t, u_t)$ has a another positive effect on off-policy generalization. Since $l_t$ includes the agents actions the distribution $P^{\pi_0}(l,u)$ might be arbitrarily far from a distribution in  $\{P^\pi(l,u)\}_{\pi \in \Pi}$. This is problematic because our low-capacity influence predictor may be unable to generalize well to large distribution shifts.

\begin{restatable}{proposition}{dd}\label{prop:distance}
Let $\pi_0$ be the exploratory policy used to collect the dataset. Then, for all $\pi \in \Pi$ 
\begin{equation*}
\small
\begin{split}
    KL(P^{\pi_0}(d_t, u_t) || P^{\pi}(d_t, u_t)) \leq  KL(P^{\pi_0}(l_t, u_t) || P^{\pi}(l_t, u_t))
\end{split}
\end{equation*}
\end{restatable}
\begin{proof} 
Found in Appendix \ref{ap:proofs}.
\end{proof}

In fact, if actions and d-sets are only weakly-coupled the distributions may be very close. In the extreme case, if d-sets are causally independent with respect to the agent's actions, the influence sources can be considered exogenous \cite{Boutilier96UAI,chitnis2020exogenous} and the two distributions are equivalent. 
\begin{restatable}{corollary}{equivalent}
Let $a_{0:t}$ be the past sequence of actions from $0$ to $t$. If $P(d_t| \text{do}(a_{o:t}) ) =  P(d_t)$ for all $d_t$ and $a_{0:t}$. Then, for any $\pi_1 \neq \pi_2 : \pi_1,\pi_2 \in \Pi$ we have $P^{\pi_1}(d_t, u_t) = P^{\pi_2}(d_t, u_t)$ even though $P^{\pi_1}(l_t, u_t) \neq P^{\pi_2}(l_t, u_t)$.
\label{prop:independence}
\end{restatable}
\begin{proof}
Found in Appendix \ref{ap:proofs}.
\end{proof}

Of course, in some domains, certain instantiations of the influence sources may only occur when very specific action sequences are followed. In such cases, a finite dataset collected with a random uniform policy will not be sufficient to fully cover the joint space of influence sources and d-sets, and thus the AIP may be unable to generalize well.  If this happens, assuming a GS is available during training, the obvious solution would be to retrain the AIP every certain number of policy updates. Nonetheless, we argue that in most cases,
such as in the two environments explored in this paper, 
the situation will be close to that of Corollary \ref{prop:independence}.
\subsection{Sufficiently similar training conditions}
\label{sec:sufficiently_similar}
Up to this point we have discussed the importance of training accurate AIPs.
Nonetheless, here we present a different perspective: to what extent is it possible to achieve near-optimal performance with inaccurate AIPs? As a matter of fact, it is very common in real-life engineering to model individual components in isolation without considering whether they belong to a larger system. Here we will fall short of giving a complete answer but we will make some observations to suggest that agents might not always require the most accurate AIPs. A view that is also supported by our experiments.

First, 
if the IALS can produce at least a few observation sequences that are similar to the ones generated by the true environment, an agent with sufficient capacity will be able to learn from those, and perform well in the real world. 
Given that one of the premises of our method is the need for accurate LS, even if the influence predictor was entirely random, the chances of getting just a small percentage of useful experiences may be in fact quite high. 
Our second argument, is that different influence distributions might still lead to the same, or very similar, best-response \citep{Becker03AAMAS}. That is, even if the agent is trained on an inaccurate IALS, some of the strategies that the agent will develop might be transferable and could also be useful when followed under real trajectories \citep{lazaric2012transfer}.

\section{Experiments}\label{sec:experiments}



The performance of our method is empirically evaluated on two benchmark domains: traffic control and warehouse commissioning. The goal of the experiments is to:
(1) Study whether we can reduce training times by replacing the GS with the IALS. (2) Measure the impact of the AIP accuracy on the learning process and the agent's final performance.
(3) Investigate the memory needs of the AIP when the agent's memory is finite.
\subsection{Experimental setup}
Agents are trained separately with PPO \cite{schulman2017proximal} on (1) the global simulator (GS), (2) the influence-augmented local simulator (IALS) with an AIP trained offline on a dataset collected from the GS, (3) an IALS that uses an untrained AIP to make predictions (untrained-IALS).
To measure the agent's performance, training is interleaved with periodic evaluations on the GS. The results are averaged over 5 runs with different random seeds.
\begin{figure}
\begin{center}
\centerline{\includegraphics[width=.65\columnwidth]{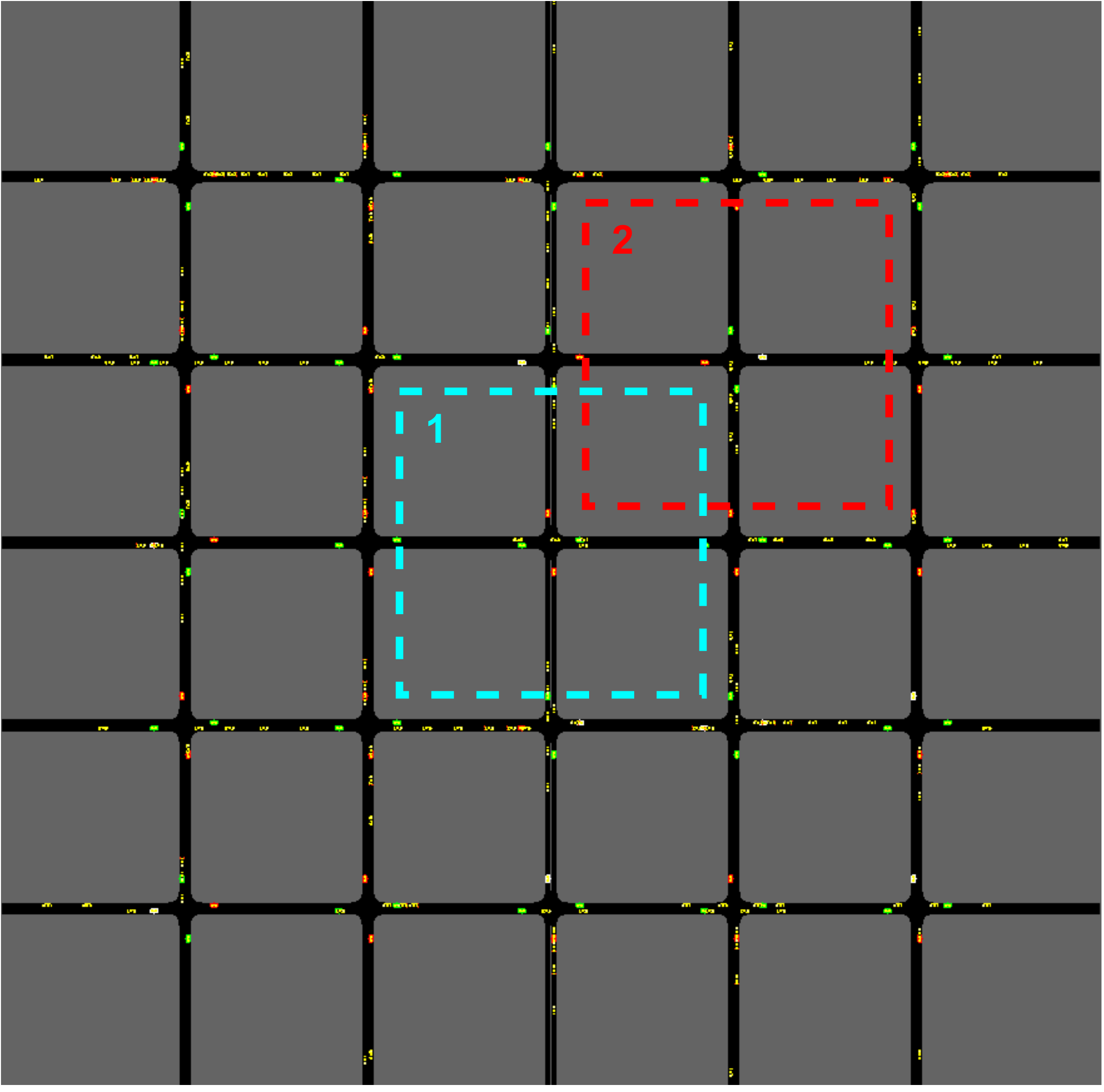}}
\caption{A screenshot of the traffic environment.}
\label{fig:trafficgrid}
\end{center}
\vskip-0.3in
\end{figure}
\subsection{Traffic control}\label{sec:traffic}
Figure \ref{fig:trafficgrid} shows a grid-like traffic network composed of 25 intersections. The agent controls the traffic lights at one of the intersections only. The rest of the traffic lights are controlled by fixed actuators that use sensors to adapt to the traffic. The policies used in this experiment for the actuated traffic light controllers were extensively optimized by \citet{wu2017flow}. We train agents separately on the two intersections highlighted in Figure \ref{fig:trafficgrid}. The goal is to maximize the average speed of cars within the intersection. The agent can only observe cars inside the dashed boxes. 
\subsubsection{GS, LS, AIP and D-SET}
The GS  and LS are built using Flow \citep{wu2017flow} and SUMO \citep{SUMO2018}. The GS simulates the entire traffic grid (Figure \ref{fig:trafficgrid}) while the LS only models the local neighborhood of the intersection being trained (Figure \ref{fig:local_sim} in Appendix \ref{ap:local_sim}). The influence sources $u_t$  are binary variables indicating whether or not a car will be entering the simulation from each of the four incoming lanes at the current timestep. The AIP $\hat{I}_\theta$ is a feedforward neural network that we train offline on a dataset of $(d_t, u_t)$ pairs collected from the GS. The  d-set $d_t$ is a length 37 binary vector encoding the location of cars along the four incoming lanes. Traffic light information is not included in $d_t$ to prevent confounding\footnote{A visual example of a spurious correlation appearing in the
traffic problem is given in Appendix \ref{ap:example}.}(Section \ref{sec:policy_dependence}). Since the two intersections highlighted in figure \ref{fig:trafficgrid} are influenced differently by the rest of the traffic network we trained separate AIPs for each of them.
\subsubsection{Results}
\begin{figure}
\begin{center}
\centerline{\includegraphics[width=1.0\columnwidth]{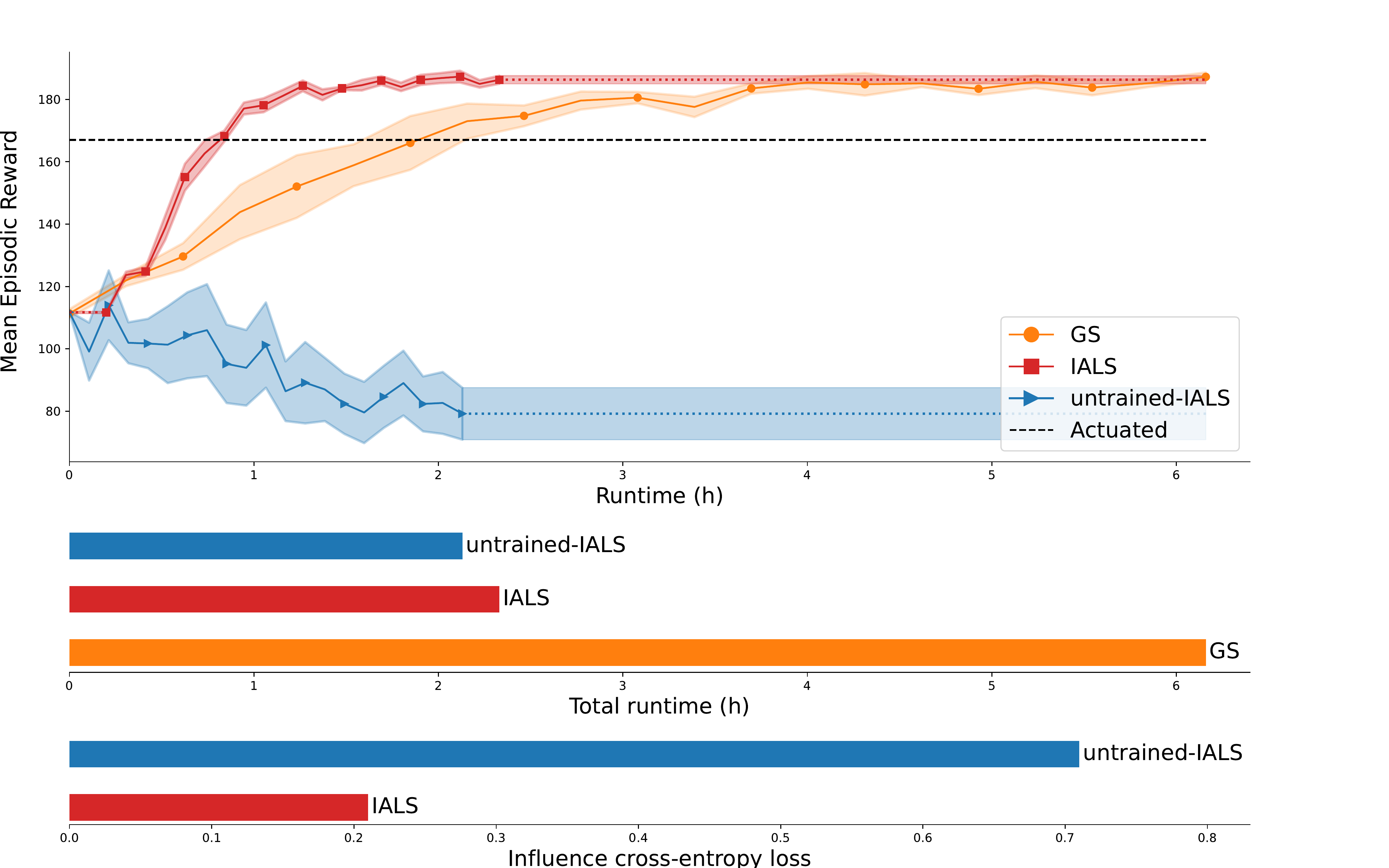}}
\end{center}
\vskip -0.2in
\caption{\textbf{Top:} Learning curves of agents trained with the GS, the IALS and the untrained-IALS on intersection 1 (Figure \ref{fig:trafficgrid}) as a function of wall-clock time.
\textbf{Middle:} Total runtime of training for 2M training steps on the three simulators. \textbf{Bottom:} Cross entropy loss for the trained and untrained AIPs.}
\vskip -0.1in
\label{fig:traffic_curves}
\end{figure}
The plot at the top of Figure \ref{fig:traffic_curves} are the learning curves of agents trained with the GS, the IALS, and the untrained-IALS to control the traffic lights at intersection 1 (Figure \ref{fig:trafficgrid})\footnote{
Results for intersection 2 are provided in Appendix \ref{ap:intersection2}.}. The plot shows the mean episodic reward as a function of real wall-clock time. Agents are trained for 2M timesteps on all three simulators. The dotted horizontal lines at the end of the red and blue curves show the agent's final performance. The short horizontal line at the beginning of the red curve represents to the AIP's training time.  The black horizontal line indicates the performance of the actuated traffic light controller. The two bar charts at the bottom show the total training time when using each of the three simulators, and the AIP's accuracy with and without training. The results suggest that policies trained on the IALS (red) can match the performance of those trained on the GS (orange) in about $1/3$ of the total training time, despite the IALS is not as accurate as the GS. This is in line with our hypothesis in Section \ref{sec:sufficiently_similar}. Similar influence distributions, $I(u_t|l_t) \approx \hat{I}_\theta(u_t|l_t)$, may lead to the same or very similar local best response. More experiments exploring this phenomenon are provided in Appendix \ref{ap:similar_training_results}. On the other hand, since the distribution $P^\pi(l_t, u_t)$ induced by the untrained AIP is very different from the true distribution, as evidenced by the high cross entropy loss (blue bar bottom chart), agents trained on the untrained-IALS (blue) perform much worse.  

\subsection{Warehouse commissioning}
A team of $36$ robots (blue) need to fetch the items (yellow) that appear with probability $0.02$ on the shelves (dashed black lines) of the warehouse in Figure \ref{fig:warehouse}. Each robot has been designated a $5 \times 5$ square region and can only collect the items that appear on the shelves at the edges. The regions overlap so that each of the $4$ item shelves in a robot's region is shared with one of its $4$ neighbors. The blue robots have been programmed to go for the oldest item in their region. The purple robot that is inside the region highlighted by the red box, on the other hand, still needs to be trained. This robot receives as observations a bitmap encoding its own location and a set of $12$ binary variables that indicate whether or not each of the $12$ items within its region needs to be collected. The purple robot, however, cannot see the location of the other robots even though all of them are directly or indirectly influencing it through their actions.
\begin{figure}
\begin{center}
\centerline{\includegraphics[width=.65\columnwidth]{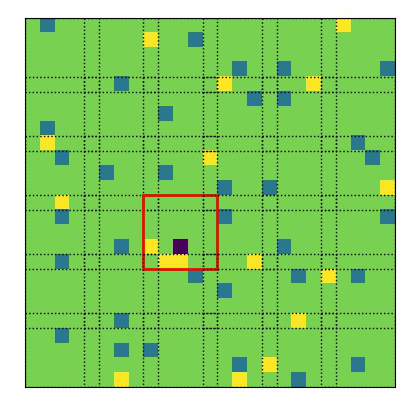}}
\vskip -0.2in
\caption{A screenshot of the warehouse environment.}
\label{fig:warehouse}
\end{center}
\vskip -0.4in
\end{figure}
\subsubsection{GS, LS, AIP and D-SET}
The GS simulates the entire warehouse (Figure \ref{fig:warehouse}), while the LS models only the $5 \times 5$ square region delimited by the red box (Figure \ref{fig:local_sim} in Appendix \ref{ap:local_sim}). The influence sources $u_t$ encode the location of the four neighbor robots. The AIP is a GRU \citep{Cho2014Learning} that we train offline on a dataset of $(d_t, u_t)$ pairs collected from the GS. If the AIP predicts that any of the neighbor robots is at one of the $12$ cells within the red box and there is an active item on that cell, that item is removed and the purple robot can no longer collect it.  The d-set $d_t$ includes the history of the 12 item variables and 12 additional binary variables encoding whether or not the controlled robot was (is) at one of the item locations. The latter variables are meant to differentiate between an item that is gone because the controlled robot collected it from an item that was picked up by the neighbor robots. The rest of variables in $l_t$ (i.e. the robot's history of locations) are unnecessary for predicting $u_t$, and thus susceptible of becoming confounders (Section \ref{sec:policy_dependence}).

\subsubsection{Results}
The plot at the top of Figure \ref{fig:warehouse_curves} shows the learning curves of the warehouse robot as a function of real wall-clock time when trained on the GS and the two local simulators, IALS and untrained-IALS. Agents are trained for 2M timesteps on all three simulators. The dotted horizontal lines at the end of the red and blue curves show the agent's final performance. The short horizontal line at the beggining of the red curve represents the AIP's training time. The two bar charts at the bottom show the total training time when using each of the three simulators, and the AIP's accuracy with and without training. Again, we see that robots trained on the IALS (red) are able to reach the same performance as those trained on the GS (orange) in about $1/3$ of the total training time despite the IALS is only approximate. Moreover, robots trained on the untrained-IALS (blue) perform reasonable well on the GS. Although the frequency at which items disappear with the untrained-IALS differs very much from that of the true environment, the basic strategy on how to collect items can still be learned. These results further confirm our hypothesis that inaccurate simulators may, in some cases, render effective policies (Section \ref{sec:sufficiently_similar}). More experiments experiments exploring this phenomenon are provided in Appendix \ref{ap:similar_training_results}.
\begin{figure}[tb]
\vskip -0.1in
\begin{center}
\centerline{\includegraphics[width=\columnwidth]{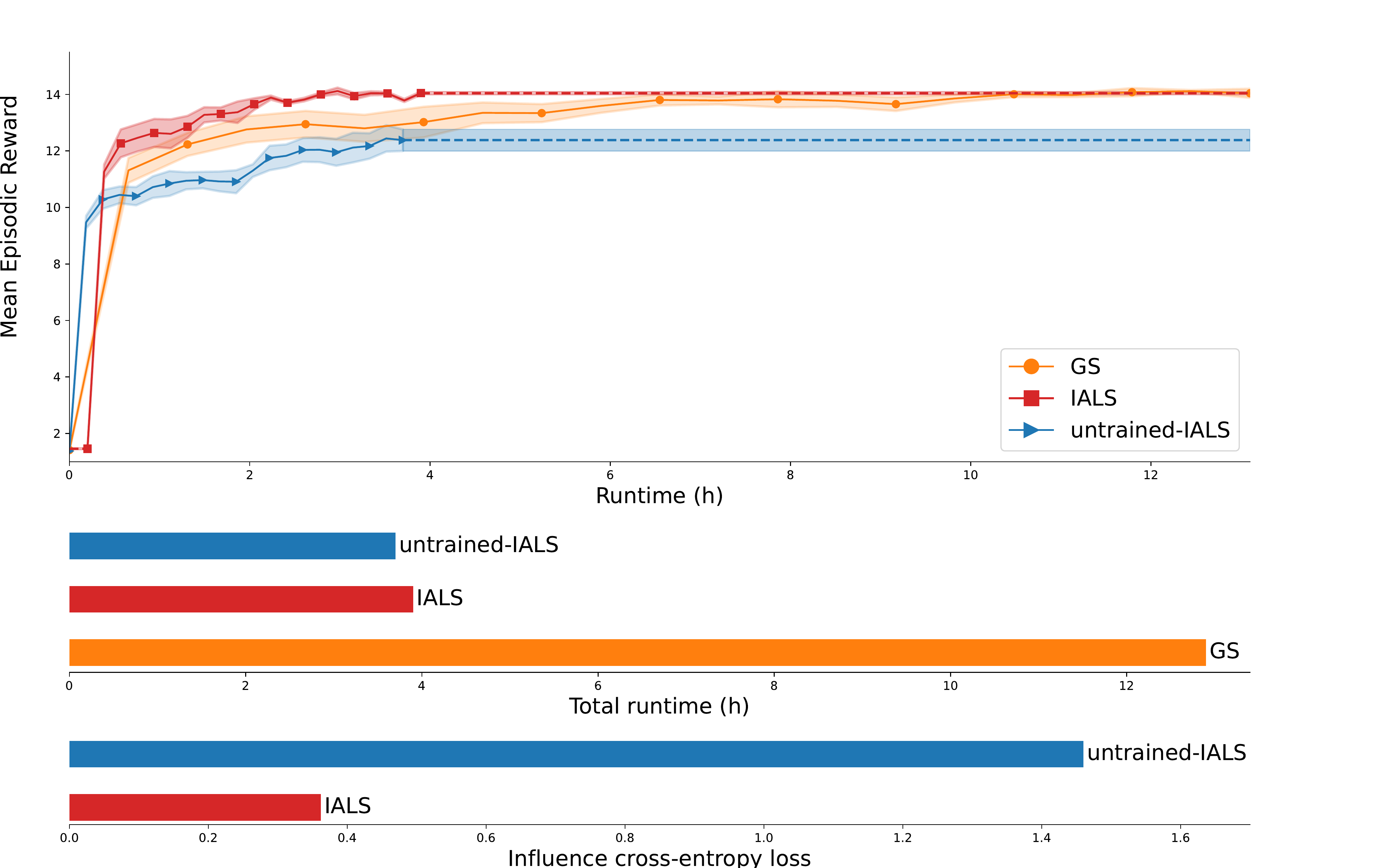}}
\caption{\textbf{Top:} Learning curves of agents trained with the GS, the IALS and the untrained-IALS on the the warehouse environment as a function of wall-clock time. 
\textbf{Middle:} Total runtime of training for 2M steps on the three simulators. \textbf{Bottom:} Cross entropy loss for the trained and untrained AIPs.}
\label{fig:warehouse_curves}
\end{center}
\vskip -0.4in
\end{figure}
\subsection{Finite memory agents and AIP history dependence}\label{sec:exp_history_dependence}
Here we investigate whether our theoretical result from Theorem \ref{thm:history_dependence} also holds in practice. We want to show that when the agent's memory is finite, meaning it can only access memories from $k$ timesteps in the past, an influence predictor which conditions on the same history length is sufficient. We test this on the warehouse domain. To make the need for memory more evident, we modify the environment so that items always disappear from the robot's region after exactly 8 timesteps.
We first train, AIPs with and without memory. We call the resulting simulators M-IALS and NM-IALS respectively. A histogram showing for how many timesteps items are active before disappearing under each of the simulators is shown in Figure \ref{fig:history_dependence}. As expected, while the former can reach an accuracy of $100\%$ (items always remain for 8 timesteps with the M-IALS), the spectrum is much wider for the NM-IALS. This is because the latter can only estimate the marginal distribution $P^\pi(u_t|o_t)$. Then, we train agents with (M) and without memory (NM) on the M-IALS and the NM-IALS. The results for all four combinations are shown in Figure \ref{fig:history_dependence}. As indicated by the theory, the plot shows that when agents have no memory AIPs may condition only on the current observation (red). In such cases, the extra level of detail that a more accurate history-dependent AIP can provide is wasted (green). In contrast, when agents can distinguish observations from one another based on their memories from the past, AIPs that make predictions by looking at the history are fundamental. This is evidenced by the gap between the blue and orange curves.
\begin{figure}[tb]
\vskip -0.2in
\begin{center}
\centerline{\includegraphics[width=\columnwidth]{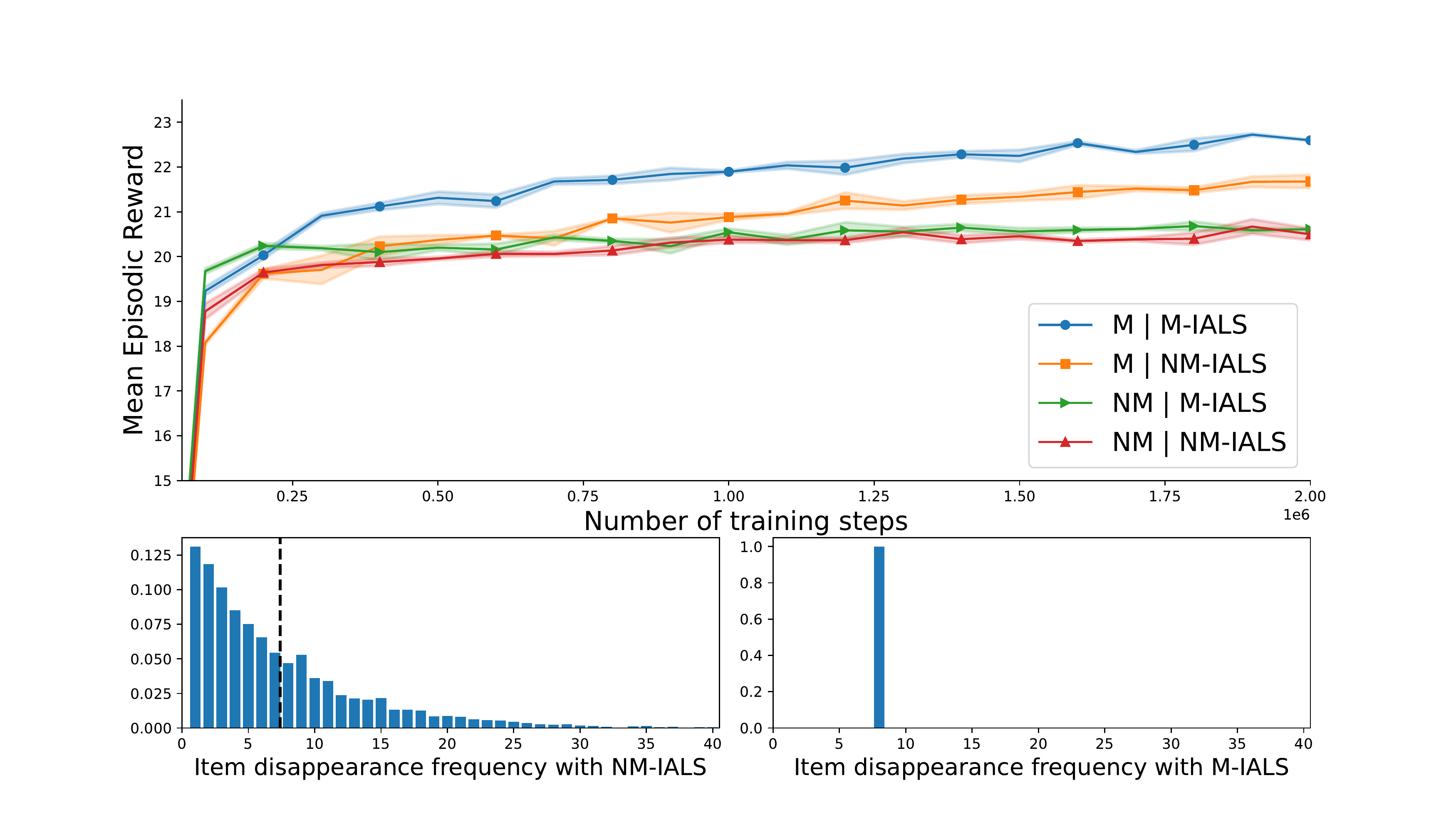}}
\vskip -0.2in
\caption{\textbf{Top:} Learning curves of agents with (M) and without memory (NM) trained on M-IALS and NM-IALS. \textbf{Bottom:} item disappearance frequencies with NM-IALS and with M-IALS.}
\label{fig:history_dependence}
\end{center}
\vskip -0.4in
\end{figure}
\section{Conclusion}\label{sec:conclusion}
This paper has offered a practical solution to allow the application of Deep RL methods to large systems, where performing exhaustive simulations can not be afforded. We focused on domains where, although the agent only interacts with a local portion of the environment, it is influenced by the global dynamics. The main idea was to replace the computationally inefficient global simulator by a lightweight version that only models the agent's local problem. However, as we showed in our experiments, directly doing this sometimes translates in a distribution shift on the agent's experience that yields poor performing policies. A good simulator needs to account for the interactions between the local region and the global dynamics. The results of our experiments suggested that by combining a pretrained influence predictor with the local simulator, we could speed up the learning process considerably while matching the performance of agents trained on the global simulator. Moreover, we analyzed the consequences of training influence predictors on data distributions that are different from those the predictor sees when deployed and resolved that, when possible, the human designer should remove from the input those variables that are irrelevant for predicting the influence sources.  
Finally, in line with our theoretical results in Section \ref{sec:history_dependence}, our last experiment revealed that the agent's memory capacity limits the memory needs of the influence predictor. 
\section*{Acknowledgements}
\newlength{\tmplength}
\setlength{\tmplength}{\columnsep}
This project received funding from the European Research
\begin{wrapfigure}{r}{0.3\columnwidth}
    \vspace{-8pt}
    \hspace{-14pt}
    \includegraphics[width=0.35\columnwidth]{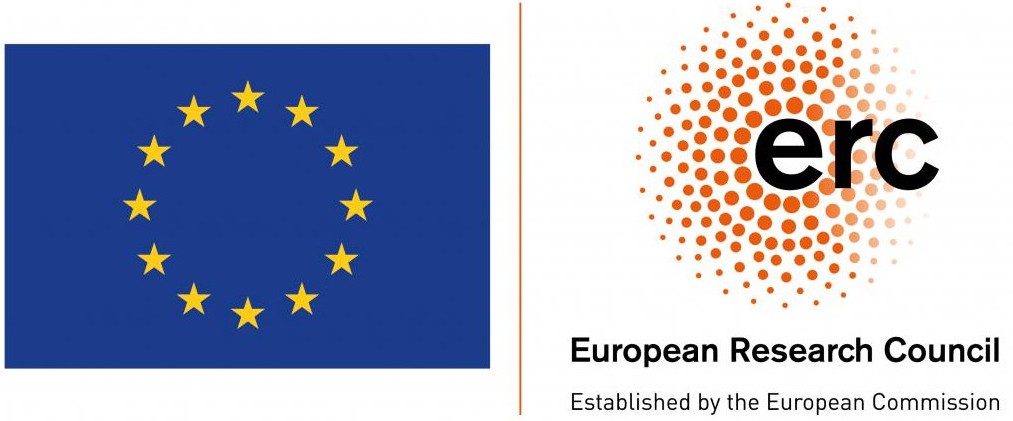}
\end{wrapfigure}Council (ERC) 
under the European Union's Horizon 2020  research 
and innovation programme (grant agreement No.~758824 \textemdash INFLUENCE).
\setlength{\columnsep}{\tmplength}

\bibliography{bibliography}
\bibliographystyle{icml2022}

\appendix
\onecolumn

\section{Proofs}\label{ap:proofs}
\renewcommand\qedsymbol{$\square$}

\history*
\begin{proof}
We will prove that, when conditioning on the last $k$ elements of the AOH, $h_{-k:t}$, the action-value function $\bar{Q}^{\bar{\pi}}(h_{-k:t},a_t)$, can be expressed in terms of an influence predictor that conditions on $l_{-k:t}$, $\bar{I}(u_t|l_{-k:t})$.  If this is true, then $Q$ estimates can be obtained by sampling from the finite memory IALS (with influence predictor $\bar{I}(u_t|l_{-k:t})$) and the agent's optimal policy (in the class of finite memory policies of the form $\bar{\pi}(a_t|h_{-k:t})$) can be found via policy iteration \cite{SuttonBarto98}. Note that, finite memory policies have been shown to perform arbitrarily bad compared to their full memory counterparts \cite{Singh94ICML}. Moreover, empirical results have revealed that standard RL methods for MDPs may fail when state representations are not Markovian \cite{Littman94memoryless}. Nonetheless, here we are only concerned with whether or not we can find the same finite memory policies when using influence predictors that condition only on $h_{-k:t}$ as those found with influence predictors that condition on the full AOH. That is, independently of the effectiveness of the RL method we use or the absolute performance of the policies we find.

Given the full AOH, the action-value function $Q^\pi(h_t,a_t)$ of a local-POMDP satisfies the Bellman equation and can be expressed as
\begin{equation}
    Q^\pi(h_t,a_t) =  \sum_{x_t}P(x_t|h_t)\dot{R}(x_t,a_t)+
    \sum_{l_t}P(l_t|h_t) \sum_{x_{t+1}} P(o_{t+1}| l_t, a_t) \sum_{a_{t+1}}\pi(a_{t+1}|h_{t+1})Q^\pi(\langle h_t, a_t, o_{t+1} \rangle,a_{t+1}),
    \label{eq:qpomdp}
\end{equation}
where and $P(x_t|h_t)$ is the probability of the local states $x_t$ given the previous AOH, and from \eqref{eq:IALM}
\begin{equation}
    P(o_{t+1}|l_t, a_t) = \sum_{x_{t+1}} \dot{O}(o_{t+1}|x_{t+1}) \sum_{u_t}   \dot{T}(x_{t+1}|x_t, a_t, u_t)I(u_t|l_t)
    \label{eq:IALM_thm1}
\end{equation}
If, instead, the agent's policy $\bar{\pi}$, and in turn the $\bar{Q}^{\bar{\pi}}$ function, can condition only on the last $k$ elements in the AOH $h_{-k:t}$, we have 
\begin{equation*}
    \bar{Q}^{\bar{\pi}}(h_{-k:t}, a_t) = \sum_{h_{0:k}}P^{\bar{\pi}}({h_{0:k}|h_{-k:t}}) Q^{\bar{\pi}}(h_t = \langle h_{0:k}, h_{-k:t} \rangle,a_t),
\end{equation*}
where $h_{0:k}$ is the AOH from $0$ to $k$ and the above expression is just the expected Q value  given $h_{-k:t}$. Then, using \eqref{eq:qpomdp}
\begin{align*}
    \bar{Q}^{\bar{\pi}}(h_{-k:t}, a_t) &= \sum_{h_{0:k}}P^{\bar{\pi}}(h_{0:k}| h_{-k:t})\sum_{x_t}P(x_t| h_{0:k}, h_{-k:t})\dot{R}(x_t,a_t) \\ +  
    &\sum_{h_{0:k}}P^{\bar{\pi}}(h_{0:k}|h_{-k:t})\sum_{l_t}P(l_t|h_{0:k}, h_{-k:t}) \sum_{o_{t+1}} P(o_{t+1}| l_t,a_t)\sum_{a_{t+1}}\bar{\pi}(a_{t+1}|h_{-k+1:t+1})\bar{Q}^{\bar{\pi}}(h_{-k+1:t+1},a_{t+1})
\end{align*}
and from \eqref{eq:IALM_thm1} we have
\begin{align*}
     \sum_{h_{0:k}}P^{\bar{\pi}}(h_{0:k}| h_{-k:t})\sum_{l_t}P(l_t|h_{0:k}, h_{-k:t}) \sum_{o_{t+1}} P(o_{t+1}| &l_t,a_t) \\ =  
    \sum_{o_{t+1}}\sum_{x_{t+1}} \dot{O}(o_{t+1}|x_{t+1}) \sum_{u_t} &\sum_{l_t} \dot{T}(x_{t+1}|x_t, a_t, u_t)I(u_t|l_t)\sum_{h_{0:k}}P^{\bar{\pi}}(h_{0:k}| h_{-k:t})P(l_t|h_{0:k}, h_{-k:t}).
\end{align*}
Then, using the rule of conditional probability, the last summation can be simplified as
\begin{align*}
    \sum_{h_{0:k}}P^{\bar{\pi}}(h_{0:k}| h_{-k:t})P(l_t|h_{0:k}, h_{-k:t}) = \sum_{h_{0:k}}P^{\bar{\pi}}(l_t,h_{0:k}|h_{-k:t}) 
    = P^{\bar{\pi}}(l_t|h_{-k:t})
\end{align*}
and thus
\begin{align*}
   \sum_{l_t}\dot{T}(x_{t+1}|x_t, a_t, u_t)I(u_t|l_t)P^{\bar{\pi}}(l_t|h_{-k:t}) &= \sum_{l_{0:k},l_{-k:t}}\dot{T}(x_{t+1}|x_t, a_t, u_t)I(u_t|l_{0:k},l_{-k:t})P^{\bar{\pi}}(l_{0:k},l_{-k:t}|h_{-k:t}) \\
   &= \sum_{l_{0:k},l_{-k:t}}\dot{T}(x_{t+1}|x_t, a_t, u_t)I(u_t|l_{0:k},l_{-k:t})P^{\bar{\pi}}(l_{0:k}|l_{-k:t})P^{\bar{\pi}}(l_{-k:t}| h_{-k:t}) \\
    \text{since } P^{\bar{\pi}}&(l_{0:k}|l_{-k:t}, h_{-k:t}) = P^{\bar{\pi}}(l_{0:k}|l_{-k:t}) \text{ because $l_{-k:t}$ is the only parent of $h_{-k:t}$} \\
   &= \sum_{l_{-k:t}} \dot{T}(x_{t+1}|x_t, a_t, u_t)\sum_{l_{0:k}}P^{\bar{\pi}}(u_t,l_{0:k}|l_{-k:t})P^{\bar{\pi}}(l_{-k:t}| h_{-k:t}) \\
   &= \sum_{l_{-k:t}}\dot{T}(x_{t+1}|x_t, a_t, u_t) \bar{I}^{\bar{\pi}}(u_t|l_{-k:t})P^{\bar{\pi}}(l_{-k:t}| h_{-k:t})
\end{align*}

Hence, putting all together,

\begin{equation*}
\small
\begin{split}
    \bar{Q}^{\bar{\pi}}&(h_{-k:t}, a_t) = \sum_{h_{0:k}}P^{\bar{\pi}}(h_{0:k}| h_{-k:t})\sum_{x_t}P(x_t| h_{0:k}, h_{-k:t})\dot{R}(x_t,a_t) \\ +&  
    \sum_{o_{t+1}}\sum_{x_{t+1}} \dot{O}(o_{t+1}|x_{t+1}) \sum_{u_t}\sum_{l_{-k:t}}\dot{T}(x_{t+1}|x_t, a_t, u_t) \bar{I}^{\bar{\pi}}(u_t|l_{-k:t})P^{\bar{\pi}}(l_{-k:t}| h_{-k:t})\sum_{a_{t+1}}\bar{\pi}(a_{t+1}|h_{-k+1:t+1})\bar{Q}^{\bar{\pi}}(h_{-k+1:t+1},a_{t+1})
\end{split}
\end{equation*}

Intuitively, we see that, because the action-value function $\bar{Q}^{\bar{\pi}}(a_t, h_{-k:t})$ is just an expectation over all possible $h_t$ given $h_{-k:t}$, the distribution of $u_t$ given the full ALSH $l_t$, $I(u_t|l_t)$, is effectively ``washed out'' by the same expectation. Hence, we may as well condition the influence predictor directly on $l_{-k:t}$, $\bar{I}^{\bar{\pi}}(u_t|l_{-k:t})$. 

Finally, using the assumption $P(u_t|l_{-k:t}, a_{0:k}) = P(u_t|l_{-k:t})$ we can drop the superscript $\bar{\pi}$ from $\bar{I}^{\bar{\pi}}(u_t|l_{-k:t})$. This assumption is important because the distribution $\bar{I}(u_t|l_{-k:t})$ is generally not well-defined. Hence, if we estimate $\bar{I}^{{\bar{\pi}}_0}(u_t|l_{-k:t})$ from samples collected while following $\bar{\pi}_0$ but $P(u_t|l_{-k:t}, a_{0:k}) \neq P(u_t|l_{-k:t})$, then, when we update the policy, the marginal distribution $P^{\bar{\pi}}(a_{0:k})$ will change and the old estimates will be biased.

\end{proof}

\invariant*
\begin{figure}[h]
\begin{center}
\centerline{\includegraphics[width=.45\textwidth]{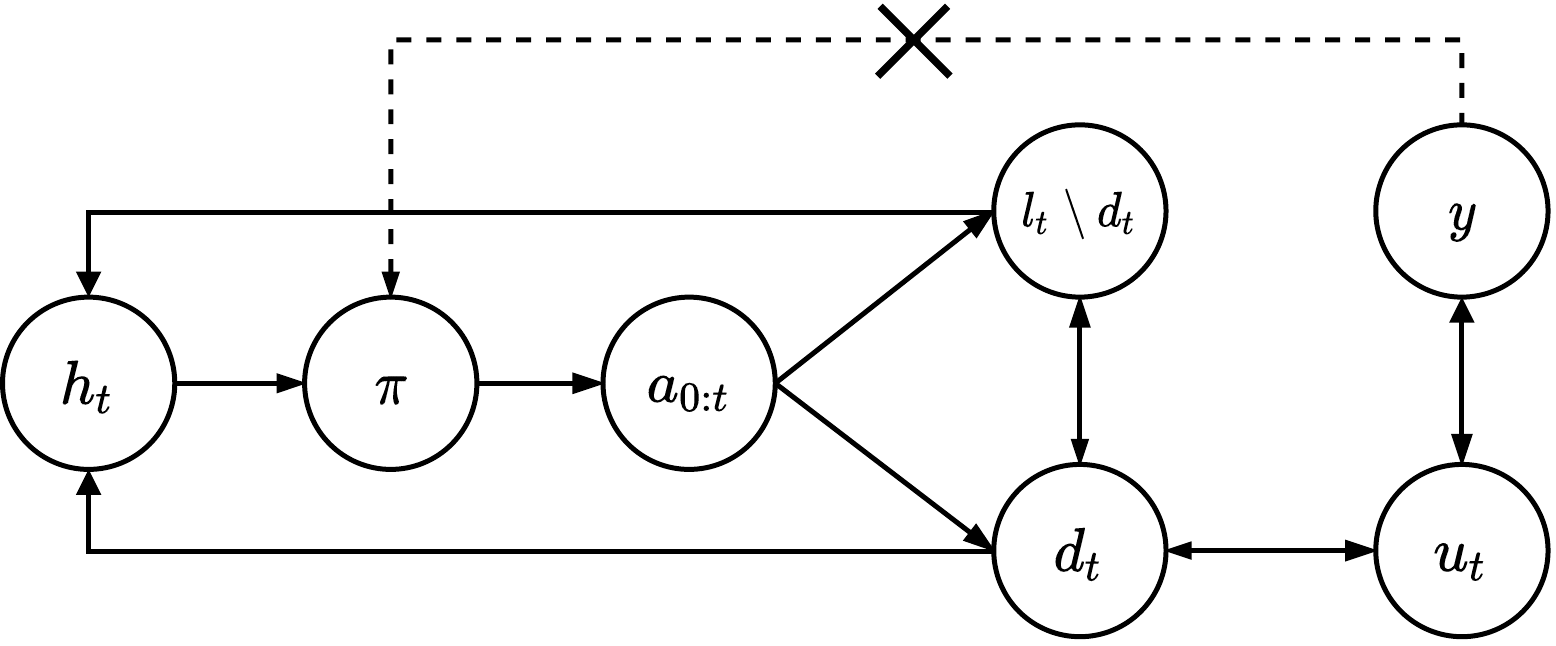}}
\caption{Graphical causal model of a local POMDP. The bidirectional arrows between two sets of variables (e.g. between $d_t$ and $u_t$) indicate that there can be variables in one set that are affected by other variables in the other set and vice-versa. As depicted by the two diagonal arrows in the middle, actions $a_{0:t}$ may or may not be included in $d_t$. The top dashed arrow connecting $y$ and $\pi$ indicates that $\pi$ cannot be a function of the non-local variables.}
\label{fig:diagram_theorem2}
\end{center}
\end{figure}
\begin{proof}
In this proof, we will make use of the graphical causal model in Figure \ref{fig:diagram_theorem2}. We know that the agent's policy $\pi$ can only influence the environment through its actions $a_{0:t}$. Moreover, from the definition of d-set (Definition \ref{def:dset}), we know that the \emph{direct path} \cite{Pearl2000Causality} that connects the past sequence of actions $a_{0:t}$ with $u_t$ must go through $d_t$. Otherwise, $(u_t \not\!\perp\!\!\!\perp d_t \setminus l_t| d_t)$ and $d_t$ would not constitute a d-set. Note that, as shown by the diagonal arrows in the middle, actions in $a_{0:t}$ may or may not be included in $d_t$. Finally, the second condition in Theorem \ref{thm:invariant} implies that there cannot be \emph{backdoor paths} \cite{Pearl2000Causality} between $\pi$ and $u_t$. Hence, if we control for $d_t$, $\pi$ and $u_t$ become conditionally independent,
\begin{equation*}
 P(u_t|d_t) = P^{\pi}(u_t| d_t) \quad \text{for all} \quad \pi \in \Pi.    
\end{equation*}
Moreover, for the above to hold, the d-set does not necessarily need to be minimal, $d_t \subseteq d^*_t$. That is, as long as the path between $l_t \setminus d_t$  and $u_t$ remains closed, $d_t$ may include as many additional (irrelevant) variables as there are in $l_t$. Note that this path is only  open when $d_t$ is not a d-set. 
\end{proof}
\dd*
\begin{proof}
\begin{align*}
 KL(P^{\pi_0}(d_t, u_t) || P^{\pi}(d_t, u_t)) &=  \sum_{d_t,u_t} P^{\pi_0}(d_t, u_t) \log \left( \frac{P^{\pi_0}(d_t, u_t)}{P^{\pi}(d_t, u_t)}\right)\\
 &=\sum_{d_t,u_t} P^{\pi_0}(u_t, d_t) \log \left(\frac{P^{\pi_0}(u_t|d_t)P^{\pi_0}(d_t)}{P^{\pi}(u_t|d_t)P^{\pi}(d_t)}\right) \\
 &= \sum_{d_t} \log \left( \frac{P^{\pi_0}(d_t)}{P^{\pi}(d_t)}\right) \sum_{u_t} P^{\pi_0}(u_t, d_t) \\
 & \text{since, from Theorem \ref{thm:invariant}, we know that } P(u_t|d_t) \text{ is invariant across all } \pi \in \Pi \\
 &= KL(P^{\pi_0}(d_t) || P^{\pi}(d_t))
\end{align*} 
Similarly, because $P(u_t|l_t)$ is also invariant across all $\pi \in \Pi$, $P^{\pi_0}(u_t|l_t) = P^\pi(u_t|l_t)$,
\begin{equation*}
   KL(P^{\pi_0}(l_t, u_t) || P^{\pi}(l_t, u_t)) = KL(P^{\pi_0}(l_t) || P^{\pi}(l_t)).
\end{equation*}
Then, since $d_t \subseteq l_t$, we can write
\begin{equation*}
    P^\pi(l_t) = P^\pi(d_t,  l_t \setminus d_t)
\end{equation*}
and, using the chain rule
\begin{align*}
   KL(P^{\pi_0}(l_t) || P^{\pi}(l_t)) &= KL(P^{\pi_0}(l_t \setminus d_t |d_t) || P^{\pi}(l_t \setminus d_t| d_t)) + KL(P^{\pi_0}(d_t) || P^{\pi}(d_t)) \\
   & \geq  KL(P^{\pi_0}(d_t) || P^{\pi}(d_t))
\end{align*}
where the last inequality holds because the KL divergence is always non-negative.
\end{proof}
\equivalent*
\begin{proof}

Follows from Proposition \ref{prop:distance}. We know that $\pi$ can only influence the environment through the actions $a_{0:t}$. Hence, 
\begin{equation*}
\small
\begin{split}
    \forall \pi_1 \neq \pi_2,d_t,a_{0:t} \quad P(d_t|do(a_{0:t})) = P(d_t) \Rightarrow P^{\pi_1}(d_t) = P^{\pi_2}(d_t) \iff KL(P^{\pi_1}(d_t) || P^{\pi_2}(d_t)) = 0
\end{split}
\end{equation*}
and from the proof of Proposition \ref{prop:distance}
\begin{equation*}
\small
    \begin{split}
     KL(P^{\pi_1}(d_t, u_t) || P^{\pi_2}(d_t, u_t)) = KL(P^{\pi_1}(d_t) || P^{\pi_2}(d_t))   
    \end{split}
\end{equation*}
which means
\begin{equation*}
\small
\begin{split}
    \forall \pi_1 \neq \pi_2,d_t,a_{0:t},u_t \quad  P(d_t|do(a_{0:t})) = P(d_t) \Rightarrow KL(P^{\pi_1}(d_t, u_t) || P^{\pi_2}(d_t,u_t)) = 0 \iff  P^{\pi_1}(d_t, u_t) = P^{\pi_2}(d_t, u_t) 
\end{split}
\end{equation*}
On the other hand, because $l_t \subseteq a_{0:t}$ then
\begin{equation*}
\small
\begin{split}
     \forall \pi_1 \neq \pi_2, l_t,a_{0:t},u_t \quad  KL(P^{\pi_1}(l_t) || P^{\pi_2}(l_t)) > 0 \iff KL(P^{\pi_1}(l_t,u_t) || P^{\pi_2}(l_t,u_t)) > 0 \iff P^{\pi_1}(l_t, u_t) \neq P^{\pi_2}(l_t, u_t) 
\end{split}
\end{equation*}

\end{proof}
\section{Example: Spurious Correlations in the Traffic Domain}\label{ap:example} 
Figure \ref{fig:spurious_corr} shows four screenshots taken from the SUMO simulator \cite{SUMO2018}. These capture a sequence consecutive states in the traffic domain (see Section \ref{sec:traffic} for more details about this environment). The agent's local region is depicted by the red dashed box. The small cyan box on the right of every screenshot indicates the location of an influence source. As shown in the screenshots, when traffic is dense, lines of cars are formed along the incoming lanes in front of the traffic lights. This situation leads to the appearance of spurious correlations between the traffic lights and the influence sources. In particular, Figure \ref{fig:spurious_corr} reveals that three timesteps after the traffic lights on the horizontal lanes are switched to green a new car appears at the influence source. Although there is clearly no direct relationship between these two events, if this pattern occurs often enough, as it is the case under a uniform random policy, the AIP might pick it up and use it as a shortcut. That is, the AIP might learn to predict that a car will show at the influence source exactly three timesteps after the lights are switched two green. This would indeed be effective at the beginning of training, since traffic jams are common under poor performing policies, but may result in highly inaccurate predictions when policies are further improved. As explained in Section \ref{sec:policy_dependence} the problem above can be prevented if d-sets, rather than full AOHs, are fed into the AIP.
\begin{figure}[ht]
\begin{center}
\centerline{\includegraphics[width=\textwidth]{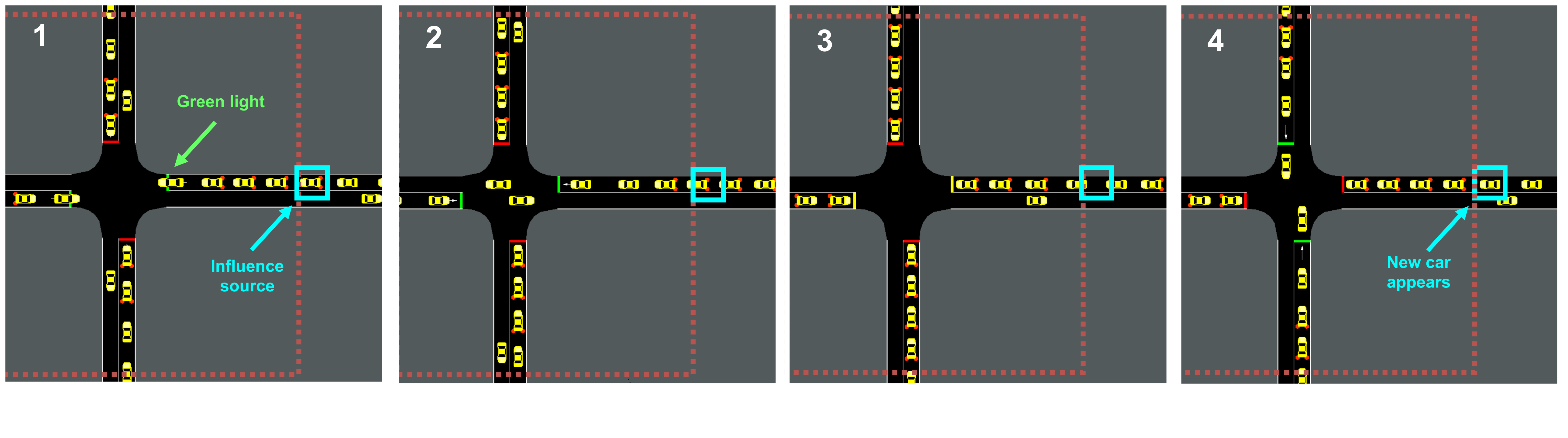}}
\caption{Four screenshots capturing a sequence of states that may occur under a uniform random policy. The agent's local region is depicted by the red dashed box. The small cyan box on the right of every screenshot indicates the location of an influence source. The screenshots reveal that, when the traffic is dense, three timesteps after the traffic lights on the horizontal lanes are switched to green a new car appears at the influence source. This is clearly a spurious correlation as there is not direct relationship between the traffic lights and the influence sources.}
\label{fig:spurious_corr}
\end{center}
\end{figure}
\section{Local Simulators}\label{ap:local_sim}
\begin{figure}[h!]
\label{fig:results}
\centering
  \begin{subfigure}[b]{0.5\textwidth}
   \hspace{30mm}
    \includegraphics[width=0.5\textwidth]{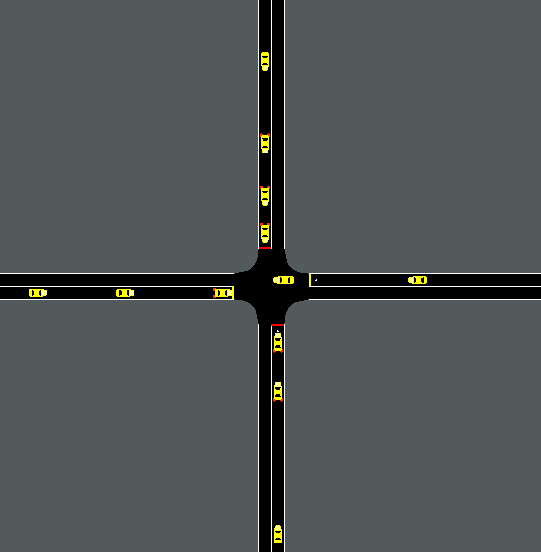}
    \label{fig:breakout}
  \end{subfigure}
  \begin{subfigure}[b]{0.45\textwidth}
    \includegraphics[width=0.58\textwidth]{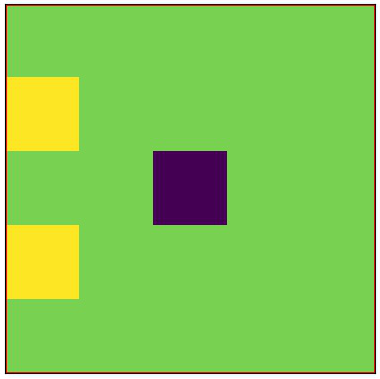}
    \label{fig:pong}
  \end{subfigure}
\caption{A screenshot of the local simulators for the  traffic control (left) and warehouse (right) environments} 
\label{fig:local_sim}
\vspace{-10pt}
\end{figure}

\section{Results: traffic control intersection 2}\label{ap:intersection2}

\begin{figure}[H]
\begin{center}
\centerline{\includegraphics[width=.6\columnwidth]{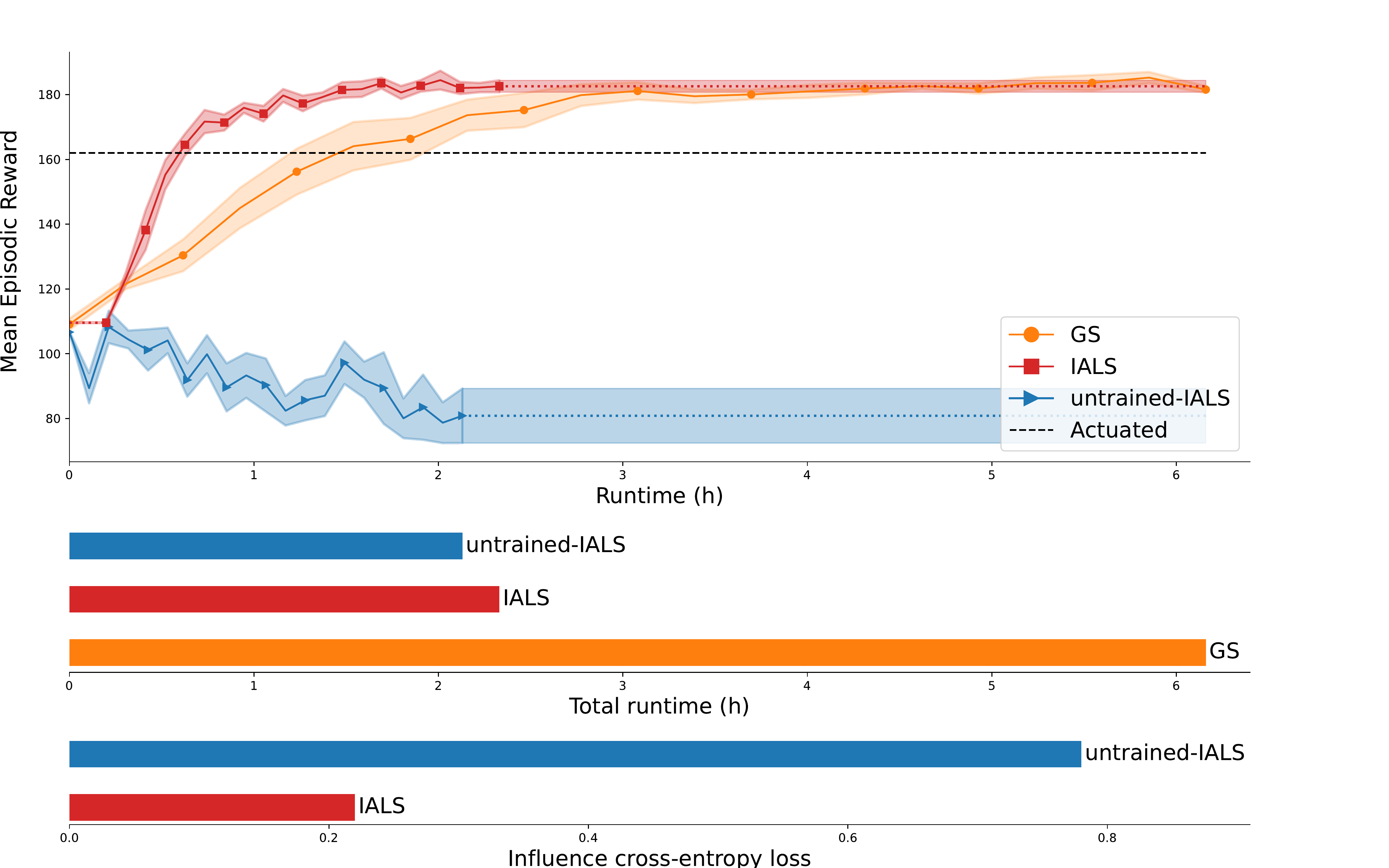}}
\end{center}
\caption{\textbf{Top:} Learning curves of agents trained with the GS, the IALS and the untrained-IALS on intersection 2 (Figure \ref{fig:trafficgrid}) as a function of wall-clock time. 
The dotted horizontal lines show the final performance of the agents after 2M timesteps of training. The dotted horizontal line at the begining of the red curve corresponds to the AIP training time. 
\textbf{Middle:} Total runtime of training for 2M training steps on the three simulators. \textbf{Bottom:} Cross entropy loss for the trained and untrained AIPs.}
\label{fig:traffic_curves2}
\end{figure}

\section{Results: sufficiently similar training conditions }\label{ap:similar_training_results}
Here we further explore our research question in Section \ref{sec:sufficiently_similar}. To what extent can agents trained on inaccurate IALS achieve similar performance to those trained on the GS? To shed some light on this this complicated question, we introduce a new type of AIP which represents the influence sources by a fixed marginal probability $P(u_t)$ independent of the ALSH. We call the resulting simulator F-IALS (F for fixed). 

\paragraph{Traffic results:}

Figure \ref{fig:traffic_curves} shows the performance of agents trained on the F-IALS with $P(u_t) = 0.1$ (F-IALS 0.1) and $P(u_t) = 0.5$ (F-IALS 0.5). The bar plot at the bottom of Figure \ref{fig:traffic_curves} shows that the cross-entropy loss is much higher when modeling the influence sources with the marginal distribution $P(u_t)$ (F-IALS 0.1 and F-IALS 0.5) than when modeling them with the learned influence predictor $P(u_t|l_t)$ (IALS). This suggests that there is a non-trivial relationship between ALSH and influence sources, and in principle
\begin{equation}
    KL(I(u_t|d_t) || \hat{I}_\theta(u_t| d_t)) < KL(I(u_t|d_t) || P(u_t) = 0.1) < KL(I(u_t|d_t) || P(u_t) = 0.5)
\end{equation}
Nonetheless, agents trained on the F-IALS 0.1 reach the same (intersection 1) or similar (intersection 2) level of performance as those trained on the IALS. This is in line with our hypothesis in section \ref{sec:sufficiently_similar}. If real trajectories are somewhat likely under a random influence predictor, an agent with sufficient capacity will be able to learn from them and perform well on the true environment. In fact, given that the probability used for the inflow of vehicles entering the GS is also 0.1, the chances of generating traffic patterns with the F-IALS (0.1) similar to those of the GS are very high. On the other hand, when setting the probability of cars entering the LS to a less sensible $P(u_t) = 0.5$ agents trained on the F-IALS (0.5) can no longer match the same performance level. 
\begin{figure}[H]
\begin{center}
\centerline{\includegraphics[width=.6\columnwidth]{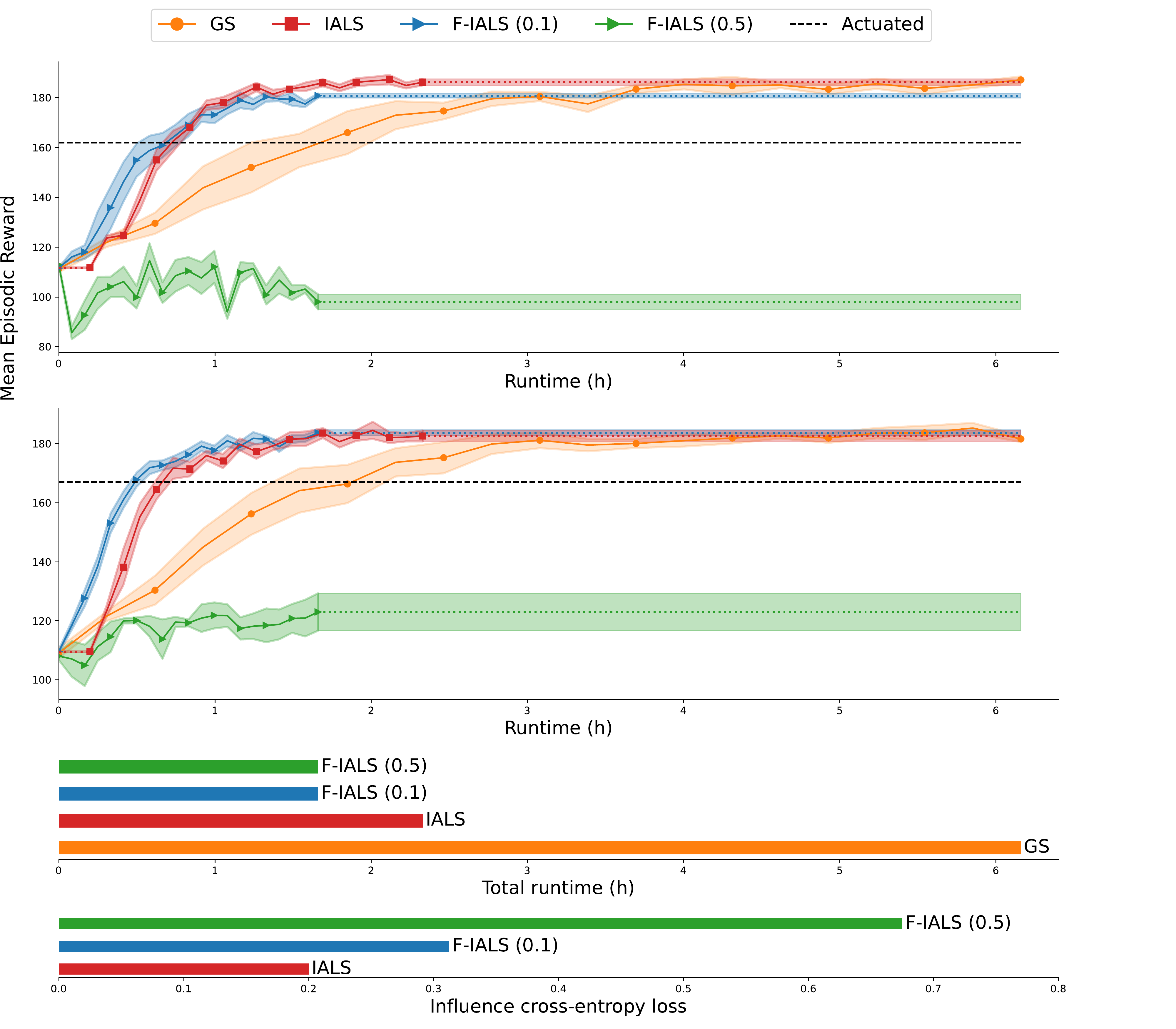}}
\end{center}
\caption{\textbf{Top:} Learning curves of agents trained with the GS, the IALS and the F-IALS on the the two intersections highlighted in Figure \ref{fig:trafficgrid} as a function of wall-clock time. The dotted horizontal lines show the final performance of the agents after 2M timesteps of training. The short horizontal line at the beginning of the red curve corresponds to the AIP training time. \textbf{Second from the bottom:} Total runtime of training for 2M training steps on the three simulators. \textbf{Bottom:} Cross entropy loss evaluated at intersection 1 for the three AIPs (values are very similar for intersection 2).}
\label{fig:traffic_sufficient}
\end{figure}
\begin{figure}[H]
\begin{center}
\centerline{\includegraphics[width=.6\columnwidth]{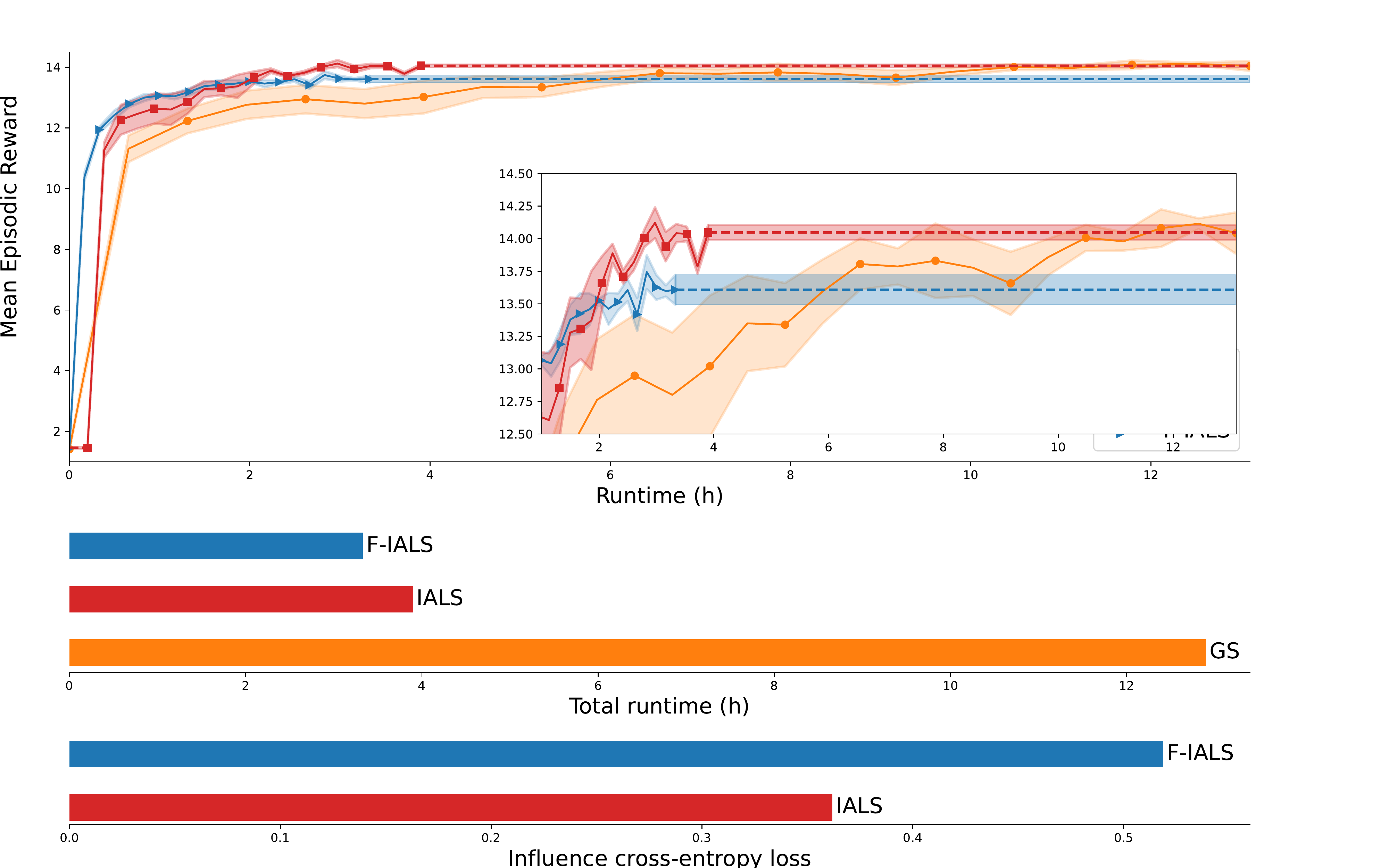}}
\end{center}
\caption{\textbf{Top:} Learning curves of agents trained with the GS, the IALS and the F-IALS on the the warehouse commissioning task as a function of wall-clock time. Zoom in version of the same chart showing the performance difference between IALS and F-IALS. The dotted horizontal lines show the final performance of the agents after 2M timesteps of training. The dotted horizontal line at the beginning of the red curve corresponds to the AIP training time. \textbf{Second from the bottom:} Total runtime of training for 2M steps on the three simulators. \textbf{Bottom:} Cross entropy loss for the two AIPs.}
\label{fig:warehouse_sufficient}
\end{figure}

\paragraph{Warehouse results:}
In this environment the fixed influence source probabilities is set to an estimate $\hat{P}(u_t)$ of true value $P^{\pi_0}(u_t)$, which we approximated empirically from $N = 10K$ samples collected from the GS while following $\pi_0$. Even so, the cross-entropy of the F-IALS is again higher than that of the learned influence predictor $\hat{I}_\theta(u_t|l_t)$, indicating again that influence sources $u_t$ and ALSHs are strongly coupled with one another and that there is a non-trivial relationship between them,
\begin{equation}
    KL(I(u_t|d_t) || \hat{I}_\theta (u_t| d_t))  < KL (I(u_t|d_t) || \hat{P}(u_t))
\end{equation}
Despite being less accurate, agents trained on the F-IALS can also perform well on the true environment. Yet, they do not reach the same level of performance as those obtained with the GS and the IALS (see zoom in version in Figure \ref{fig:warehouse_curves}). Even though, the basic strategy on how to collect items can be learned from the F-IALS, the simulator does not provide the extra level of detail needed to learn better policies. That is, a consistent pattern that is present in both the GS and the IALS by which items that have been active the longest are more likely to disappear next. With this, agents can learn to not go for an item when the chances that a neighbor robot will get there first are high.

All in all we see that, in certain situations, inaccurate influence predictors can still provide a fair amount of useful experiences for the agent to perform well on the true environment. However, in domains with complicated dynamics, such as our warehouse environment, the best policies can only be obtained when simulations provide the extra level of detail that only accurate influence predictors are able to deliver.

\section{Implementation Details}\label{ap:imp_details}

\paragraph{Approximate influence predictor: } Due to the sequential nature of the problem, rather than feeding the full past history every time we make a prediction, we use a recurrent neural network (RNN) \citep{hochreiter1997long, Cho2014Learning} and  process observations one at a time,
\begin{equation}
\small
    P(u_t|l_t) \approx \hat{I}_\theta(u_t| \hat{h}_{t-1}, o_t) = F_{\text{rnn}}(\hat{h}_{t-1},o_t, u_t),
    \label{eq:RNN}
\end{equation}
where we use $\hat{h}$ to indicate that the history $h$ is embedded in the RNN's internal memory. 

Given that we generally have multiple influence sources $u_t = \langle u^1_t \dots u^M_t\rangle$, we need to fit $M$ separate models $\hat{I}_{\theta_m}$ to predict each of the $M$ influence sources. In practice, to reduce the computational cost, we can have a single network with a common representation module for all influence sources and output their probability distributions using $M$ separate heads. This representation assumes that the influence sources are independent of one another,
\begin{equation}
\small
    I(u_t|l_t) = \prod_{m=0}^M P(u^m_t|l_t),
\end{equation}
which is true for the two domains we study in this paper.

\paragraph{Practical implications of Theorem \ref{thm:history_dependence}:} One important consideration when training RNNs via backpropagation through time (BPTT) is to choose the right length for the input sequences \citep{williams1990efficient}. This determines the number of steps the network is backpropagated, and thus for how long past information can be retained. On the one hand, longer sequences will often provide more information to predict the influence sources, on the other, they will also make it harder to optimize the network. Ideally, we would like to choose just the right sequence length such that the agent cannot perceive any distribution shift in the local transitions. Assuming that in our environment $P(u_t|l_{-k:t},a_{0:k}) = P(u_t|l_{-k:t})$, from Theorem \ref{thm:history_dependence} we know that the sequence length should  be (at least) as long as that of the agent's (if these are also modelled by RNNs or as long as the number of stacked observations fed to feedforward neural networks; FNN). If $P(u_t|l_{-k:t},a_{0:k}) \neq P(u_t|l_{-k:t})$ we can still condition the AIP only on $l_{-k:t}$ but we might need to retrain the AIP every certain number of policy updates to prevent it from becoming stale (see last paragraph in the proof of theorem \ref{thm:history_dependence}; Appendix \ref{ap:proofs}).

\paragraph{Policies:} We model policies by FNNs. In the warehouse environment, since the agent needs memory to perform well, policies are fed with a stack of the last 8 observations. This architecture performed better than GRUs. In the traffic control task, on the other hand, policies are fed only with the current observation as this  seems to be sufficient to predict the influence sources.
\newpage
\section{Algorithms}\label{ap:algorithms}
\begin{algorithm}
\caption{Collect dataset with GS}
\begin{algorithmic}[1]
\State \textbf{Input:} $T$, $\pi_0$
\Comment{Global simulator and exploratory policy}
\For{$n \in \langle 0, ..., N / T \rangle$}
\State  $s_0 \gets$ reset$(T)$
\Comment{Reset global simulator}
\State $x_0 \gets s_0$ \Comment{Extract local state from global state}
\State $l_0 \gets x_0$
\Comment{Initialize ALSH with initial local state}
\State $o_{0} \sim O(o_{0} \mid s_{0})$
\Comment{Sample observation from $O$}
\State $h_0 \gets o_0$
\Comment{Initialize AOH with initial observation}
\For{$t \in \langle0, ..., T\rangle$}
\State $\langle u_t^0, ..., u_t^m \rangle \gets s_t$ \Comment{Extract influence sources from global state}
\State $D \gets \{ l_t, \langle u_t^0, ..., u_t^M \rangle \}$ 
\Comment{Append ALSH-influence-source pair to the dataset}
\State $a_t \sim \pi(h_t)$ 
 \Comment{Sample action}
\State $s_{t+1} \sim T(s_{t+1} \mid s_t, a_t)$  
\Comment{Sample next state from GS}
\State $x_{t+1} \gets s_{t+1}$ 
\Comment{Extract local state from global state}
\State $l_{t+1} \gets \langle a_t, x_{t+1} \rangle $ 
\Comment{Append action-local-state pair to ALSH}
\State $o_{t+1} \sim O(o_{t+1} \mid s_{t+1})$
\Comment{Sample observation from $O$}
\State $h_{t+1} \gets \langle a_t, o_{t+1} \rangle$ 
 \Comment{Append actions-observation pair to the AOH}
\EndFor
\EndFor
\end{algorithmic}
\label{alg:collect}
\end{algorithm}
\begin{algorithm}
\caption{Simulate trajectory with IALS}
\begin{algorithmic}[1]
\State \textbf{Input:} $\dot{T},\pi,\hat{I}_\theta$
\Comment{Local simulator, policy, and AIP}
  \State $o_0 \gets$ reset($\dot{T}$)
  \Comment{Reset local simulator}
  \State $h_0 \gets o_0$
  \Comment{Initialize history with initial observation}
  \For{$t \in \langle0, ..., T\rangle$}
  \State $a_t \sim \pi(h_t)$ 
  \Comment{Sample action}
  \For{$m \in \langle 0, ..., M\rangle$}
  \Comment{Iterate over the M influence sources}
  \State  $u_t^m \sim \hat{I}_{\theta_m}(l_t) $ 
  \Comment{Sample influence source from AIP}
  \EndFor
  \State $x_{t+1} \sim \dot{T}(x_{t+1}| x_t, a_t, u_t = \langle u^0, ..., u^M \rangle)$ 
  \Comment{Sample next local state from LS}
  \State $l_{t+1} \gets \langle a_t, x_{t+1} \rangle $ 
  \Comment{Append action-local-state pair to ALSH}
  \State $o_{t+1} \sim O(o_{t+1}|x_{t+1})$  
  \Comment{Sample observation from $O$}
  \State $h_{t+1} \gets \langle a_t, o_{t+1} \rangle $ 
  \Comment{Append action-observation pair to AOH}
  \EndFor
\end{algorithmic}
\label{alg:sample}
\end{algorithm}

\end{document}